\documentclass[11pt]{article}

\usepackage[preprint]{acl}

\usepackage{times}
\usepackage{latexsym}
\usepackage{enumitem}
\usepackage[T1]{fontenc}

\usepackage[utf8]{inputenc}

\usepackage{microtype}
\definecolor{rawgray}{RGB}{242,242,242}
\definecolor{diagblue}{RGB}{220,235,247}
\definecolor{bestorange}{RGB}{252,232,206}
\usepackage{inconsolata}

\usepackage{graphicx}
\usepackage{times}
\usepackage{epsfig}
\usepackage{graphicx}
\usepackage{amsmath}
\usepackage{amssymb}
\usepackage{amsthm}
\usepackage{booktabs}
\usepackage{xcolor}
\usepackage{tikz}
\usepackage{array}
\usetikzlibrary{positioning,arrows.meta}
\usepackage{tikz}
\usepackage{graphicx}
\usepackage{amsmath,amssymb}
\usepackage{algorithm}
\usepackage{algpseudocode}
\usepackage{graphicx}
\usepackage{amsmath,amssymb}
\usepackage{algorithm}
\usepackage{algpseudocode}
\usepackage[most]{tcolorbox}
\usepackage{listings}
\usepackage{xcolor}

\algrenewcommand\algorithmicrequire{\textbf{Require:}}
\algrenewcommand\algorithmicensure{\textbf{Output:}}
\algrenewcommand\algorithmicrequire{\textbf{Require:}}
\algrenewcommand\algorithmicensure{\textbf{Output:}}
\usepackage{amsmath,amssymb,mathtools}

\theoremstyle{plain}
\newtheorem{assumption}{Assumption}
\newtheorem{proposition}{Proposition}

\newtheorem{corollary}{Corollary}
\newtheorem*{assumption*}{Assumption}
\newtheorem*{proposition*}{Proposition}
\newtheorem*{theorem*}{Theorem}
\newtheorem*{corollary*}{Corollary}
\theoremstyle{remark}
\newtheorem{remark}{Remark}
\newtheorem*{remark*}{Remark}
\newcommand{\cmark}{\checkmark}
\newcommand{\xmark}{$\times$}

\title{Seeking Universal Shot Language Understanding Solutions}

\author{Haoxin Liu\textsuperscript{\textdagger}\thanks{Correspondence to: Haoxin Liu <hliu763@gatech.edu> and B. Aditya Prakash <badityap@cc.gatech.edu>}, Harshavardhan Kamarthi\textsuperscript{\textdagger}, Zhiyuan Zhao\textsuperscript{\textdagger}, Hongjie Chen\textsuperscript{\S},  B. Aditya Prakash\textsuperscript{\textdagger}\footnotemark[1] \\ \textsuperscript{\textdagger}Georgia Institute of Technology \quad \textsuperscript{\S}Dolby Laboratories\\
}

\begin{document}

\maketitle

\begin{abstract}
Shot language understanding (SLU) is crucial for cinematic analysis but remains challenging due to its diverse cinematographic dimensions and subjective expert judgment. While vision-language models (VLMs) have shown strong ability in general visual understanding, recent studies reveal judgment discrepancies between VLMs and film experts on SLU tasks. To address this gap, we introduce \textbf{\textsc{SLU-SUITE}}, a comprehensive training and evaluation suite containing 490K human-annotated QA pairs across 33 tasks spanning six film-grounded dimensions. Using \textsc{SLU-SUITE}, we originally observe two insights into VLM-based SLU from: the model side, which diagnoses key bottlenecks of modules; the data side, which quantifies cross-dimensional influences among tasks.
These findings motivate our universal SLU solutions from two complementary paradigms: \textbf{\textsc{UniShot}}, a balanced one-for-all generalist trained via dynamic-balanced data mixing, and \textbf{\textsc{AgentShots}}, a prompt-routed expert cluster that maximizes peak dimension performance. Extensive experiments show that our models outperform task-specific ensembles on in-domain tasks and surpass leading commercial VLMs by 22\% on out-of-domain tasks.
\end{abstract}
\vspace{-3mm}
\section{Introduction}
\label{sec:intro}
\vspace{-1mm}
\emph{Shot language} is essential for high-quality filmmaking by conveying narrative meaning, emotion, and aesthetic style~\cite{bordwell2008film,shotbench,wu2025moviebench}. Shot language understanding (SLU) is therefore crucial yet quite complicated as: (1) Broadly, SLU spans multiple cinematography aspects, including composition, shot scale, camera motion, lighting, color, and inter-shot relationships. (2) Moreover, even within a single aspect, different experts often adopt different taxonomies (e.g., different composition rules or different levels of granularity for camera motion).

While recent advancements in vision-language models (VLMs) have shown strong ability on general visual understanding, emerging studies~\cite{shotbench,camerabench,cinetechbench,chatterjeestable,wu2025refineshot} reveal notable judgment discrepancies between VLMs and film experts on SLU tasks. These findings highlight the gap and the need for \emph{universal SLU models} that align with expert knowledge, support diverse SLU tasks, and generalize to unseen scenarios. However, one major obstacle to advancing SLU is the lack of suitable data. As in Tab.~\ref{tab:dataset_compare}, existing human-annotated SLU datasets cover a narrow subset of cinematic skills and remain limited in scale~\cite{shotbench,camerabench,cinetechbench,moviecuts}, while VLM-labeled datasets~\cite{wu2025moviebench,ju2024miradata} expand coverage but suffer from quality issues, particularly on misalignment with human expert judgments.

To address these limitations, we introduce \textbf{\textsc{SLU-SUITE}}, a unified training and evaluation suite for general SLU. \textsc{SLU-SUITE} organizes data into six film-grounded dimensions: \emph{composition}, \emph{coverage}, \emph{viewpoint}, \emph{motion}, \emph{lighting}, and \emph{cuts}~\cite{bordwell2008film,brown2016cinematography,chatterjeestable}, which covers \textbf{33 tasks} with \textbf{490K} human-labeled QA pairs, and supports both controlled in-domain (ID) and out-of-domain (OOD) evaluation. Building on \textsc{SLU-SUITE}, we conduct a comprehensive empirical study that identifies the bottlenecks of current VLM-based SLU from two complementary perspectives: the \emph{model side}, which diagnoses key bottlenecks of existing VLMs, and the \emph{data side}, which quantifies cross-dimensional influences among SLU tasks. Gained insights motivate our highly efficient and effective VLM parameter-updating strategy and lead to universal SLU solutions from two complementary paradigms: \textbf{\textsc{UniShot}}, designed as a balanced generalist across cinematic dimensions, and \textbf{\textsc{AgentShots}}, designed as a collection of specialized experts for peak performance on the target dimension.

In summary, our contributions lie in four aspects:

\noindent \textbf{New Datasets and Testbed.} We introduce \textbf{\textsc{SLU-SUITE}}, the first comprehensive and large-scale human-labeled benchmark for general SLU. It encompasses 490K QA pairs across 33 tasks to cover 6 film-grounded dimensions, natively supporting both rigorous ID and OOD evaluations.

\noindent \textbf{Deep Insights.} Through empirical explorations, we suggest that the primary VLM bottleneck for SLU is \emph{semantic alignment} rather than raw visual perception. This makes LM+Connector updating more effective (better performance) and efficient (15\% fewer trainable parameters), compared with all modules updating. Furthermore, we quantify and identify that cross-dimensional transfer is broadly useful but highly \emph{asymmetric and uneven}, thus decoupling universal SLU into two distinct objectives: a balanced generalist and peak specialists.

\noindent \textbf{Data-Centric Solutions.} By these insights, we develop two universal SLU models built upon Qwen3-VL-8B using \textsc{SLU-SUITE}, powered by dedicated data-mixing strategies. \textbf{\textsc{UniShot}} is a single balanced generalist trained via \emph{dynamic-balanced data mixing}, and \textbf{\textsc{AgentShots}} is a prompt-routed expert cluster maximizing peak dimensional performance via \emph{target-aware data mixing}.

\noindent\textbf{Comprehensive Evaluation with Theoretical Analysis.} Extensive evaluations show that a single \textsc{UniShot} model surpasses an ensemble of 12 task-specific SFT models on 9/12 ID tasks, while \textsc{AgentShots} exhibits exceptional zero-shot generalization, outperforming leading commercial VLMs (Gemini-3.0-Pro) by 22\% in average relative accuracy on OOD tasks. Ablation studies further validate the effectiveness of our data-centric strategies. We further theoretically prove that large-scale, multi-source data heterogeneity, as the core characteristic of \textsc{SLU-SUITE}, is essential for training general SLU models.

We provide more related works and limitations in App.~\ref{app:related_work} and ~\ref{app:limitations}. We commit to releasing datasets and models upon paper acceptance.

\begin{table}[h]
\centering
\scriptsize
\setlength{\tabcolsep}{3pt}
\renewcommand{\arraystretch}{1.12}
\resizebox{\linewidth}{!}{%
\begin{tabular}{c c c c}
\toprule
\textbf{Human-Labeled Dataset} & \textbf{Scale} & \textbf{Training Supp.} & \textbf{OOD Eval.} \\
\midrule
ShotBench~\cite{shotbench} &
\shortstack[c]{8 tasks\\ $\sim$70K QA} & \cmark & \xmark \\
CameraBench~\cite{camerabench} &
\shortstack[c]{5 tasks\\ $\sim$150K labels} & \cmark & \xmark \\
CineTechBench~\cite{cinetechbench} &
\shortstack[c]{7 tasks\\ $\sim$0.7K QA} & \xmark & N/A \\
MovieCuts~\cite{moviecuts} &
\shortstack[c]{1 task\\ $\sim$110K clips} & \cmark & \xmark \\
\midrule
\textbf{SLU-SUITE (Ours)} &
\shortstack[c]{\textbf{33 tasks}\\ \textbf{$\sim$490K QA}} & \cmark & \cmark \\
\bottomrule
\end{tabular}%
\vspace{-3mm}
}

\caption{\textbf{Comparison with representative human-labeled SLU datasets.}
\textbf{Scale} reports the number of tasks and the approximate data size.
\textbf{Training Supp.} indicates whether official splits support fine-tuning/training.
\textbf{OOD Eval.} indicates whether the benchmark supports controlled out-of-domain evaluation by design (e.g., unseen tasks), compared with the common in-domain (ID) evaluation. \textbf{Our SLU-SUITE} covers a broader set of tasks, provides large-scale human-labeled data, and explicitly supports both ID and OOD evaluations to advabce general SLU.}
\label{tab:dataset_compare}
\vspace{-6mm}
\end{table}
\begin{table*}[h]
\centering
\footnotesize
\setlength{\tabcolsep}{4pt}
\renewcommand{\arraystretch}{1.12}
\resizebox{\textwidth}{!}{%
\begin{tabular}{@{}c >{\centering\arraybackslash}p{15.2cm} c c@{}}
\toprule
\textbf{High-Level Dimension} & \textbf{Detailed Tasks} & \textbf{Num. Sources} & \textbf{Num. Samples} \\
\midrule
Lighting &
\textbf{4} tasks: \textit{SourceCondition}; \textit{Style}; \textit{Attribute}; \textit{ColorPalette} &
2 & 925 \\
Composition &
\textbf{3} tasks: \textit{CompositionRule}; \textit{CompositionPattern}; \textit{VisualWeightPlacement} &
3 & 4{,}841 \\
Viewpoint &
\textbf{6} tasks: \textit{AngleBasic}; \textit{AngleExtended}; \textit{AngleFinegrained}; \textit{AngleCartoon}; \textit{Height}; \textit{HeightCartoon} &
4 & 37{,}742 \\
Coverage &
\textbf{9} tasks: \textit{ScaleBasic}; \textit{ScaleClassic}; \textit{ScaleExtended}; \textit{ScaleFinegrained}; \textit{ScaleCartoon}; \textit{ScaleHistorical}; \textit{Staging}; \textit{StagingScaleMix}; \textit{FocalLength} &
7 & 51{,}693 \\
Motion &
\textbf{10} tasks: \textit{MoveCoarse}; \textit{MoveCompoundA}; \textit{MoveCompoundB}; \textit{MoveCaptioning}; \textit{VQAComplexity}; \textit{VQAMovement}; \textit{VQAShaking}; \textit{VQASpeed}; \textit{VQAPresence}; \textit{VQAMixedType} &
5 & 282{,}297 \\
Cuts &
\textbf{1} task: \textit{InterShotCutType} &
1 & 109{,}827 \\
\midrule
\textbf{Total} &
\textbf{33} tasks &
\textbf{11} sources & \textbf{487{,}325} QA pairs \\
\bottomrule
\end{tabular}%
}
\vspace{-2mm}
\caption{
\textbf{SLU-SUITE summary by high-level dimension.}
Each row groups task variants under one dimension.
The \textbf{Detailed Tasks} lists task variants in that dimension (count in bold).
We define a task variant by a specific option set (label space/taxonomy); thus, semantically similar tasks from different sources are treated as different variants when their option sets differ.
\textbf{Num. Sources} denotes the number of contributing datasets, and \textbf{Num. Samples} denotes the number of labeled QA pairs.
Overall, SLU-SUITE includes 33 task variants from 11 datasets, totaling 487{,}325 human-labeled QA pairs and covering 53{,}569 unique images and 154{,}451 unique videos after deduplication.
}
\label{tab:suite_overview}
\vspace{-4mm}
\end{table*}
\vspace{-1mm}
\section{\textsc{SLU-SUITE}: Training and Evaluation Suite for General SLU}
\label{sec:suite}
\vspace{-1mm}
\subsection{The Datasets}
\vspace{-1mm}
To mitigate the scarcity of data in shot language understanding (SLU), we build \textbf{\textsc{SLU-SUITE}}, a unified suite designed for \emph{general} SLU training and evaluation.
Compared with existing human-labeled cinematography datasets (Tab.~\ref{tab:dataset_compare}), \textsc{SLU-SUITE} covers a broader set of tasks and explicitly supports both in-domain (ID) and out-of-domain (OOD) evaluation. As detailed in Tab.~\ref{tab:suite_overview}, \textsc{SLU-SUITE} spans six high-level dimensions, including \textbf{33 tasks} with  \textbf{$\sim$490K human-labeled samples ($\sim$53.6K image and $\sim$433.7K video QA pairs)}.

\noindent\textbf{{Organizing 33 Tasks into 6 Film-Grounded Dimensions.}}
Shot language understanding data are highly heterogeneous: different sources use different label spaces (e.g., coarse vs.\ fine-grained shot scale), some tasks are multi-label (option combinations), and some tasks are binary or free-form (caption).
Training and evaluating on a flat list of tasks makes it hard to control coverage, prevent taxonomy confusion, and analyze transfer.

We therefore organize \textsc{SLU-SUITE} into \textbf{six high-level dimensions} that follow standard cinematography controls used in film analysis and on-set communication (e.g., framing/composition, scale and lensing, viewpoint, camera movement, lighting/color, and editing/cuts), consistent with structured filmmaking taxonomies~\cite{bordwell2008film,brown2016cinematography,chatterjeestable}. As summarized in Tab.~\ref{tab:suite_overview}, each dimension contains multiple tasks: \underline{\emph{Lighting}} (\textbf{4} tasks) captures illumination and color style, including source/condition, lighting style/attributes, and palette cues (\textit{SourceCondition}, \textit{Style}, \textit{Attribute}, \textit{ColorPalette}). \underline{\emph{Composition}} (\textbf{3} tasks) captures how subjects and visual mass are arranged in the frame, including different layout templates (\textit{Composition*}) and visual placement (\textit{VisualWeightPlacement}). \underline{\emph{Viewpoint}} (\textbf{6} tasks) describes camera placement relative to the subject and scene, including angle taxonomies (\textit{Angle*}) and camera height (\textit{Height*}). \underline{\emph{Coverage}} (\textbf{9} tasks) models what a shot covers and how coverage is constructed, including shot scale taxonomies (\textit{Scale*}), staging-based coverage patterns (\textit{Staging}), mixed staging--scale labels (\textit{StagingScaleMix}), and lens-related cues (\textit{FocalLength}). \underline{\emph{Motion}} (\textbf{10} tasks) covers both movement taxonomies and movement attributes: coarse vs.\ compound movements (\textit{MoveCompound*}), motion description (\textit{MoveCaption*}), and attribute-focused VQA such as speed, shaking, and complexity (\textit{VQA*}). \underline{\emph{Cuts}} (\textbf{1} task) focuses on editing transitions within a shot-level context (\textit{InterShotCutType}).

Note that we consider semantically similar tasks as different detailed tasks when their option sets differ (e.g., \textit{ScaleBasic} vs.\ \textit{ScaleFinegrained}).

\noindent\textbf{Constructing via 11 Sources and 4 Stages.}
We build \textsc{SLU-SUITE} through a multi-source pipeline.
We first identify 11 dataset sources that satisfy two requirements:
(i) annotations are produced by humans (or verified by humans), rather than generated by LLMs;
(ii) the sources jointly cover the six dimensions above, and each dimension is supported by at least two sub-tasks or has sufficient data volume (e.g., $\ge$50K).
For each source, we collect the corresponding media and labels, and remove samples with missing files, broken links, corrupted content, or decoding failures.
We then filter sensitive content using \textit{Gemini-3.0-Flash}.
Finally, we convert all sources into a unified QA format, including both classification-style QA and captioning-style QA.
For classification-style QA, we provide explicit candidate options in the question and allow multi-label answers. See  details in App.~\ref{app:slu_suite_protocol}.
\vspace{-1mm}
\subsection{Evaluation Protocols}
\vspace{-1mm}
\noindent\textbf{Mitigating media-level leakage.}
Due to curating from multiple sources, 
we perform de-duplication across images and videos to ensure that evaluation media do not appear in the training set.
After de-duplication, we obtain \textbf{about 53K unique images and 154K unique video clips}.
\noindent\textbf{ID and OOD evaluation.}
We evaluate generalization in two settings: \emph{ID}, where test tasks are included in training, and \emph{OOD}, where test tasks are excluded from training. 

For \underline{OOD evaluation}, we use \textbf{11} classification-style QA tasks from \emph{CineTechBench}~\cite{cinetechbench} and \emph{CameraBench}~\cite{camerabench}, two existing SLU datasets for test.  For the remaining \textbf{22} tasks, we select \textbf{12} tasks with more than 3K samples and split each of them into \textbf{80\%} training and \textbf{20\%} \underline{ID evaluation}.
We add all samples from the other \textbf{10} tasks to the \underline{training set}. See detailed data partitions in App.~\ref{app:data_splits}.
To further avoid cross-source leakage, we remove any OOD test sample whose media appears in the training set.
For simplicity, we use classification-style QA for both training and testing, and supplement the training set with captioning-style QA. Finally, we use \textbf{410K} QA pairs for training and \textbf{50K} QA pairs for evaluation.
\vspace{-1mm}
\section{Empirical Explorations on SLU-SUITE}
\label{sec:exploration}
\vspace{-1mm}
We investigate two key research questions at the model and data levels on the constructed \textsc{SLU-SUITE} to motivate our solution designs.
\vspace{-1mm}
\subsection{RQ1: The bottleneck of VLMs for SLU?}
\label{sec:rq1_bottleneck}
\vspace{-1mm}
Although general VLMs are trained on massive multimodal corpora, they still perform poorly on SLU tasks, even though films are a common form of media in their training data. This naturally raises a key question: what is the bottleneck of VLMs for SLU? We compare LoRA adaptation on the \textbf{vision encoder}, \textbf{multimodal connector}, \textbf{language model (LM)} and \textbf{all}.
To isolate module choice from cross-dimensional transfer, each SLU dimension is trained and evaluated independently under the same backbone, split, LoRA rank (\textbf{32}), and optimization setup; only the adapted modules differ.
Additional setup details are given in App.~\ref{app:rq1_details}.

\begin{table}[t]
  \centering
  \small
  \setlength{\tabcolsep}{4pt}
  \renewcommand{\arraystretch}{1.08}
  \caption{\textbf{Results of updating different modules of VLM (Qwen3-VL-8B) for SLU.}
  Each row trains and tests on a single SLU dimension, thereby removing cross-dimensional transfer effects.
  \textbf{LM+Connector} achieves the best average performance while using about \(15\%\) fewer trainable parameters than \textbf{All}.
  Vision-only adaptation on viewpoint resulted in model collapse, i.e., incapable of solving the OOD tasks; see App.~\ref{app:rq1_details}}.
  \label{tab:ablation-lora-dim}
  \vspace{-2mm}
  \resizebox{\linewidth}{!}{%
  \begin{tabular}{c|cccccc}
    \hline
                     & \multicolumn{6}{c}{LoRA Target}                                      \\ \hline
Dimension            & Base  & Connector & Vision & LM    & LM+Connector & All            \\ \hline
Lighting             & 0.533 & 0.533     & 0.524  & 0.610 & \textbf{0.661} & 0.616          \\
Composition          & 0.062 & 0.412     & 0.277  & 0.546 & \textbf{0.623} & 0.579          \\
Viewpoint            & 0.500 & 0.725     & --     & 0.750 & 0.780          & \textbf{0.801} \\
Coverage             & 0.457 & 0.609     & 0.627  & 0.619 & \textbf{0.667} & 0.633          \\
Motion               & 0.602 & 0.647     & 0.605  & 0.713 & \textbf{0.724} & 0.721          \\
Cuts                 & 0.230 & 0.340     & 0.325  & 0.330 & \textbf{0.380} & 0.365          \\
\hline
Avg. (all 6)         & 0.397 & 0.544     & --     & 0.595 & \textbf{0.639} & 0.619          \\
Avg. (shared 5)      & 0.377 & 0.508     & 0.471  & 0.563 & \textbf{0.611} & 0.583          \\
\# Best dimensions   & 0     & 0         & 0      & 0     & \textbf{5}     & 1              \\
Trainable Params (M) & 0     & 2.3       & 15.4   & 87.3  & \textbf{89.6}  & 105.0          \\ \hline
\end{tabular}}
\vspace{-5mm}
\end{table}

As shown in Tab.~\ref{tab:ablation-lora-dim} (1) Updating the \textbf{LM} gives much larger gains than updating only the connector or only the vision encoder, and is only slightly worse than updating the whole VLM. (2) Surprisingly, updating \textbf{LM+Connector} even outperforms updating all modules, while using 15\% fewer trainable parameters. In contrast, updating only the vision encoder sometimes causes the model to collapse, i.e., fail to complete the task.

These original insights suggest that advanced VLMs already capture the visual signals needed for SLU well, and updating the vision tower can even hurt performance; the main bottleneck likely lies on the language side, such as aligning visual signals with cinematographic terms and learning the decision boundaries for SLU tasks. We therefore adopt \textbf{LM+Connector} LoRA fine-tuning as the default VLM post-training strategy.
\vspace{-2mm}
\subsection{RQ2: Effects across SLU dimensions?}
\label{sec:rq2_transfer}
Given that \textsc{SLU-SUITE} consists of six different dimensions, a key question for building universal solutions is how these dimensions interact with each other. 
Let \(\mathcal{D}\) denote the six SLU dimensions.
For each source dimension \(i \in \mathcal{D}\), we fine-tune a model using data from \(i\) and evaluate it on every target dimension \(j \in \mathcal{D}\) to quantify the transfer effect.
We denote the resulting accuracy by \(A_{i\rightarrow j}\), the pretrained baseline by \(A_{\mathrm{base}\rightarrow j}\), and the target-only performance by \(A_{j\rightarrow j}\). Notably, since each dimension is trained with all its training data, the matrix reflects \emph{practical transfer utility in the natural SLU setting}, where some dimensions are low-resource---Lighting has fewer than 1K samples although aggregated from four existing datasets. 
As shown in Fig.~\ref{fig:transfer_matrix} and Tab.~\ref{tab:transfer-summary}, we have two key insights:

\noindent\textbf{Insight 1: Cross-dimensional supervision is broadly reusable.}
Most off-diagonal source--target pairs improve over the base model, indicating that different SLU dimensions share substantial reusable cues.  For low-resource targets such as \textsc{Lighting}, several auxiliary dimensions even outperform target-only training.
\emph{SLU should therefore not be treated as a collection of isolated dims.}

\noindent\textbf{Insight 2: Transfer utility is strongly target-dependent and asymmetric.}
Auxiliary data are not equally useful for all targets.
For example, \textsc{Lighting}, \textsc{Composition}, \textsc{Viewpoint}, and \textsc{Motion} benefit from multiple external sources, whereas \textsc{Coverage} receives much weaker transfer overall. 
The relation is also asymmetric: \textsc{Coverage} shows strong transfer-out helpfulness but weak transfer-in receptiveness.  This suggests that the best recipe for one target is not necessarily the best for balanced performance across all dims.

In short, these suggest that \emph{(1) balanced performance across dims and peak performance for target dim might be distinct objectives, and (2) carefully designed data mixing strategies are necessary.}
\vspace{-8mm}
\noindent\begin{center}
\fbox{
\parbox{\linewidth}{
\textbf{Key Implications for SLU Model Building.}

RQ1 shows that the main bottleneck of VLMs for SLU is the \textbf{language side}, leading us to adopt \textbf{updating LM+Connector} as the default VLM post-training strategy, which is more efficient and effective than updating all modules.

RQ2 shows that cross-dimensional supervision is \textbf{broadly useful but highly uneven and asymmetric}.
This motivates us to decouple universal SLU with \textbf{two distinct objectives}:
\textbf{(1)} a \emph{generalist} that maintains balanced performance across dimensions, and
\textbf{(2)} a set of \emph{specialists} that maximize performance for the corresponding dimensions.
It also motivates our focus toward \textbf{data-centric solutions}, by designing cross-dimensional data mixing strategies dedicated for different objectives.
}}
\end{center}
\vspace{-2mm}

\begin{figure}[t]
    \centering
    \includegraphics[width=\linewidth]{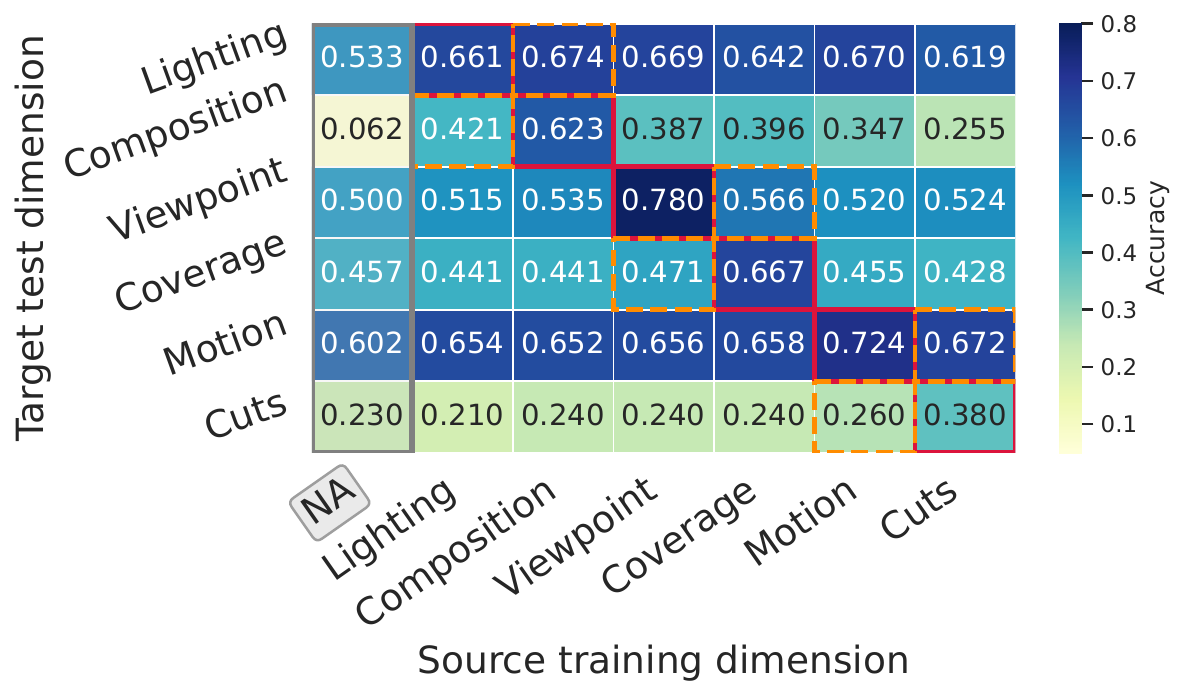}
    \vspace{-4mm}
    \caption{\textbf{Cross-dimensional transfer matrix.}
    Each column denotes the source training dimension and each row denotes the target test dimension.
    The \textit{NA} column corresponds to initial performance of the VLM backbone (Qwen3-VL-8B).
    Transfer is broadly useful, but clearly uneven across targets.}
    \label{fig:transfer_matrix}
    \vspace{-5mm}
\end{figure}

\section{Proposed Universal SLU Models}
\label{sec:solutions}
\vspace{-1mm}
Based on these insights, we pursue two solutions for general SLU: \textbf{\textsc{UniShot}} aims to achieve \emph{balanced performance across all six dimensions within a single model}, while \textbf{\textsc{AgentShots}} aims to deliver \emph{the best performance on a queried SLU task through prompt-routed specialists}.
\vspace{-1mm}
\subsection{\textsc{UniShot}: Single Generalist via Dynamic-Balanced Data Mixing}
\label{sec:unishot}

\noindent\textbf{Goal.}
\textsc{UniShot} learns a \emph{single balanced model} for all six SLU dimensions.

In Sec.~\ref{sec:exploration}, RQ2 shows that auxiliary supervision is broadly useful, but its value is uneven across targets.
\textsc{UniShot} therefore reweights training at the \emph{dimension level}: at each step, it samples a source dimension according to a learned distribution and then draws training instances from that dimension.
Implementation details, validation construction, and full pseudo-code are provided in App.~\ref{app:unishot_details}. We first derive a transfer prior from Fig.~\ref{fig:transfer_matrix}.
For each source--target pair, we define the normalized positive transfer gain as
\[
\begin{aligned}
\vspace{-2mm}
T_{i\rightarrow j}
&=
\left[
\frac{A_{i\rightarrow j}-A_{\mathrm{base}\rightarrow j}}
{A_{j\rightarrow j}-A_{\mathrm{base}\rightarrow j}+\epsilon}
\right]_+,
\\
\tilde{T}_{i\rightarrow j}
&=
\frac{T_{i\rightarrow j}}
{\sum_{k\in\mathcal{D}} T_{k\rightarrow j}+\epsilon},
\end{aligned}
\vspace{-2mm}
\]
where \([x]_+=\max(x,0)\) and \(\epsilon\) is a small constant.
Here, \(T_{i\rightarrow j}\) measures how much source \(i\) recovers the improvement margin of target \(j\) over the pretrained baseline, while \(\tilde{T}_{i\rightarrow j}\) normalizes this contribution across candidate sources.

Starting from a uniform distribution \(p_i^{(0)}=\frac{1}{|\mathcal{D}|}\), \textsc{UniShot} periodically evaluates the current model on a small held-out validation split and uses the resulting dimension-level score \(V_j^{(t)}\) to estimate how far target \(j\) still lags behind its target-only anchor \(A_{j\rightarrow j}\):
$ g_j^{(t)}
=
\left[
1-\frac{V_j^{(t)}}{A_{j\rightarrow j}+\epsilon}
\right]_+.$ A larger \(g_j^{(t)}\) indicates that dimension \(j\) remains under-optimized. We then compute the utility of each source dimension as $u_i^{(t)}
=
\sum_{j\in\mathcal{D}} \tilde{T}_{i\rightarrow j}\, g_j^{(t)},$ and update the sampling distribution by
\[
\vspace{-2mm}
\begin{aligned}
\hat{p}_i^{(t+1)}
&=
(1-\lambda)\frac{1}{|\mathcal{D}|}
+\lambda \frac{u_i^{(t)}}{\sum_{k\in\mathcal{D}} u_k^{(t)}+\epsilon}, \\
p_i^{(t+1)}
&=
(1-\beta)p_i^{(t)}+\beta \hat{p}_i^{(t+1)}.
\end{aligned}
\]
where \(\lambda\) balances uniform exploration and adaptive reweighting, and \(\beta\) controls update smoothness. \textsc{UniShot} therefore couples the current training state with the empirical transfer prior, repeatedly shifting probability mass toward sources that are most useful for the under-performed dimensions.
\vspace{-1mm}
\subsection{\textsc{AgentShots}: Peak Specialists via Target-Aware Data Mixing}
\label{sec:agentshots}
\vspace{-1mm}
\noindent\textbf{Goal.}
\textsc{AgentShots} maximizes \emph{peak performance} for queried SLU tasks via routed specialists.

In Sec.~\ref{sec:exploration}, RQ2 suggests that each target should use a different auxiliary recipe based on its own transfer profile to achieve peak performance. For each target dimension \(d\), we first identify the set of beneficial auxiliary sources: $\mathcal{S}_d
=
\{\, i \in \mathcal{D}\setminus\{d\}\;|\; A_{i\rightarrow d} > A_{\mathrm{base}\rightarrow d} \,\}.$
For each selected source \(i \in \mathcal{S}_d\), we define its transfer gain over the baseline as $s_{i|d} = A_{i\rightarrow d} - A_{\mathrm{base}\rightarrow d},
\qquad
q_{i|d}
=
\frac{s_{i|d}}{\sum_{k\in\mathcal{S}_d} s_{k|d}+\epsilon}.
$
This excludes harmful sources by construction and assigns more auxiliary weight to sources that are empirically more helpful. We then mix target and auxiliary data with
$p_d^{(d)} = \alpha,\qquad
p_i^{(d)} = (1-\alpha)\, q_{i|d}, \quad i\in\mathcal{S}_d,$
and \(p_i^{(d)}=0\) otherwise.
If \(\mathcal{S}_d=\emptyset\), we fall back to target-only training.
In practice, we set \(\alpha=0.7\), which keeps the target dominant while still exploiting useful cross-dimensional supervision.

We finally obtain six specialists---six LoRA adapters corresponding one-to-one to the six SLU dims---sharing the same frozen VLM backbone.

\textbf{Deployment Efficiency.}
\textsc{AgentShots} does not train a separate router.
Instead, the same frozen VLM backbone plays the router role: reads the task prompt together with short descriptions of the six LoRA experts, i.e., target dimensions' task lists in \textsc{SLU-SUITE}, and loads only the top-1 matched expert for prediction.
This preserves the same \emph{active} parameter count as \textsc{UniShot}; see App.~\ref{app:agentshots-routing}.

\textbf{Training Details.}
We adopt the same \textbf{LM+Connector} fine-tuning strategy as in Sec.~\ref{sec:rq1_bottleneck}, and perform supervised fine-tuning with the next-token prediction loss.
We train on the full dataset using 2$\times$A100 80GB GPUs, with a batch size of 8, a 32 LoRA rank, and 2 training epochs.
By default, we use Qwen3-VL-8B as the frozen VLM backbone, with about 89.6M (\textsc{UniShot}) and 537.6M (\textsc{AgentShots}) trainable LoRA parameters.


\begin{table*}[t]
  \centering
  \footnotesize
  \setlength{\tabcolsep}{4pt}
  \renewcommand{\arraystretch}{1.08}
  \caption{\textbf{ID results on 12 seen task variants.}
We compare general-purpose VLMs, an existing SLU model, our two universal solutions, and 12 task-specific Qwen3-VL-8B counterparts.  \textbf{Avg. Acc} denotes the unweighted mean across tasks. \textbf{Count \#1} reports the number of tasks on which a method ranks first.
\textbf{\textsc{UniShot}} achieves the best average accuracy (\textbf{0.759}), while \textbf{\textsc{AgentShots}} remains nearly identical (0.758).
Both of our universal models outperform the ensemble of 12 task-specific counterparts (0.740), suggesting positive transfer across seen SLU tasks.}
  \vspace{-2mm}
  \label{tab:iid-main}
  \resizebox{\textwidth}{!}{%
  \begin{tabular}{lcccccccc}
    \toprule
    & \multicolumn{4}{c}{General-purpose VLMs} & Existing SLU & \multicolumn{2}{c}{Ours} & Task-specific Ensemble \\
    \cmidrule(lr){2-5}\cmidrule(lr){6-6}\cmidrule(lr){7-8}\cmidrule(lr){9-9}
    12 In-Domain Tasks
    & Qwen3-VL-8B & Gemini-2.5-Flash & Gemini-3.0-Flash & Gemini-3.0-Pro
    & ShotVL-7B
    & \textbf{\textsc{UniShot}} & \textbf{\textsc{AgentShots}}
    & 12 task-specific SFT Qwen3-VL-8B \\
    \midrule
    CompositionRule  & 0.020 & 0.005 & 0.040 & 0.040 & 0.775 & \textbf{0.790} & 0.785 & 0.784 \\
    \addlinespace[1pt]
    ScaleExtended    & 0.145 & 0.165 & 0.125 & 0.170 & 0.115 & 0.290          & 0.255          & \textbf{0.320} \\
    ScaleHistorical  & 0.670 & 0.770 & 0.850 & 0.825 & 0.740 & \textbf{0.900} & 0.895          & 0.855 \\
    ScaleCartoon     & 0.715 & 0.675 & 0.650 & 0.660 & 0.460 & \textbf{0.940} & 0.925          & 0.875 \\
    ScaleBasic       & 0.650 & 0.465 & 0.595 & 0.555 & 0.425 & \textbf{0.875} & \textbf{0.875} & 0.865 \\
    Staging          & 0.430 & 0.165 & 0.280 & 0.295 & 0.290 & \textbf{0.785} & 0.760          & 0.685 \\
    \addlinespace[1pt]
    AngleBasic       & 0.330 & 0.420 & 0.440 & 0.485 & 0.245 & 0.940          & 0.955          & \textbf{0.960} \\
    AngleCartoon     & 0.350 & 0.515 & 0.450 & 0.545 & 0.260 & 0.690          & 0.705          & \textbf{0.710} \\
    Height           & 0.630 & 0.690 & 0.705 & 0.670 & 0.650 & \textbf{0.905} & 0.900          & 0.875 \\
    HeightCartoon    & 0.540 & 0.385 & 0.555 & 0.580 & 0.600 & 0.740          & \textbf{0.775} & 0.720 \\
    \addlinespace[1pt]
    MoveCoarse       & 0.745 & 0.650 & 0.710 & 0.585 & 0.600 & \textbf{0.850} & 0.840          & 0.840 \\
    \addlinespace[1pt]
    InterShotCutType & 0.230 & 0.215 & 0.295 & 0.270 & 0.210 & 0.400          & \textbf{0.425} & 0.390 \\
    \midrule
    Avg.\ Acc.\ (\(\uparrow\)) & 0.455 & 0.427 & 0.475 & 0.473 & 0.448 & \textbf{0.759} & 0.758 & 0.740 \\
    Count \#1 (\(\uparrow\))   & 0/12  & 0/12  & 0/12  & 0/12  & 0/12  & \textbf{7/12}  & 3/12  & 3/12 \\
    \bottomrule
  \end{tabular}%
  }
  \vspace{-2mm}
\end{table*}

\begin{table*}[t]
  \centering
  \footnotesize
  \setlength{\tabcolsep}{4pt}
  \renewcommand{\arraystretch}{1.08}
  \caption{\textbf{OOD results on 11 unseen task variants.}
OOD tasks remain within the same six high-level shot-language dimensions, but introduce unseen tasks, either new option sets (\textit{ScaleClassic}) or new task types (\textit{FocalLength}).  \textbf{Avg. Acc.} denotes the unweighted mean across tasks. \textbf{Count \#1} reports the number of tasks on which a method ranks first.
\textbf{\textsc{AgentShots}} delivers the strongest OOD generalization, achieving the best average accuracy (\textbf{0.666}) and ranking first on \textbf{7/11} task variants.
\textbf{\textsc{UniShot}} remains competitive at 0.617 and still exceeds all other solutions.}
  \vspace{-2mm}
  \label{tab:ood-main}
  \resizebox{0.9\textwidth}{!}{%
  \begin{tabular}{lccccccc}
    \toprule
    & \multicolumn{4}{c}{General-purpose VLMs} & Existing SLU & \multicolumn{2}{c}{Ours} \\
    \cmidrule(lr){2-5}\cmidrule(lr){6-6}\cmidrule(lr){7-8}
    11 Out-of-Domain Tasks
    & Qwen3-VL-8B & Gemini-2.5-Flash & Gemini-3.0-Flash & Gemini-3.0-Pro
    & ShotVL-7B
    & \textbf{\textsc{UniShot}} & \textbf{\textsc{AgentShots}} \\
    \midrule
    CompositionPattern & 0.083 & 0.208 & 0.158          & 0.525 & 0.133 & 0.433          & \textbf{0.558} \\
    \addlinespace[1pt]
    ScaleClassic       & 0.371 & 0.486 & \textbf{0.600} & 0.436 & 0.507 & 0.414          & 0.593 \\
    FocalLength        & 0.217 & 0.333 & \textbf{0.467} & 0.433 & 0.300 & 0.367          & 0.400 \\
    \addlinespace[1pt]
    AngleExtended      & 0.650 & 0.575 & \textbf{0.825} & 0.792 & 0.608 & 0.608          & 0.650 \\
    \addlinespace[1pt]
    MoveCompoundB       & 0.815 & 0.750 & 0.810          & 0.585 & 0.805 & \textbf{0.880} & \textbf{0.880} \\
    VQAMovement        & 0.366 & 0.452 & 0.570          & 0.441 & 0.301 & 0.634          & \textbf{0.645} \\
    VQAComplexity      & 0.535 & 0.555 & 0.570          & 0.355 & 0.490 & \textbf{0.800} & \textbf{0.800} \\
    VQAShaking         & 0.640 & 0.615 & 0.565          & 0.520 & 0.615 & 0.720          & \textbf{0.730} \\
    VQASpeed           & 0.510 & 0.680 & 0.605          & 0.645 & \textbf{0.730} & 0.675 & 0.705 \\
    \addlinespace[1pt]
    LightingStyle      & 0.582 & 0.591 & 0.600          & 0.691 & 0.582 & 0.627          & \textbf{0.700} \\
    ColorPalette       & 0.483 & 0.600 & 0.617          & 0.600 & 0.550 & 0.633          & \textbf{0.667} \\
    \midrule
    Avg.\ Acc.\ (\(\uparrow\)) & 0.478 & 0.531 & 0.581 & 0.548 & 0.511 & 0.617 & \textbf{0.666} \\
    Count \#1 (\(\uparrow\))   & 0/11  & 0/11  & 3/11  & 0/11  & 1/11  & 2/11  & \textbf{7/11} \\
    \bottomrule
  \end{tabular}%
  }
  \vspace{-3mm}
\end{table*}

  \vspace{-1mm}
\section{Experiments}
\label{sec:experiments}
  \vspace{-1mm}
\subsection{Main Results}
\label{sec:main-results}
  \vspace{-1mm}

We answer two questions in this section:
\textbf{(1) ID performance:} can our universal SLU solutions match or exceed strong general-purpose and task-specific models on seen SLU tasks?
\textbf{(2) OOD generalization:} can they transfer to unseen SLU task?

\noindent\textbf{Compared methods.} We compare against five existing universal baselines:
(i) \textbf{Qwen3-VL-8B}~\cite{bai2025qwen3} as the open-source backbone baseline;
(ii) \textbf{Gemini-2.5-Flash}, \textbf{Gemini-3.0-Flash}, and \textbf{Gemini-3.0-Pro} as strong commercial VLMs; and
(iii) \textbf{ShotVL-7B}~\cite{shotbench} as an existing SLU model.
For ID evaluation, representing the standard isolated-training paradigm, we further compare with \textbf{12 task-specific SFT Qwen3-VL-8B models}, each trained on a single training task and evaluated on the corresponding ID task.

\noindent\textbf{Evaluation Metrics.} All methods are evaluated under the same classification-style QA protocol.
Predictions are scored by exact match on the target answer string; for multi-label questions, all labels must be correct.
OOD tasks remain within the same six high-level shot-language dimensions, but introduce unseen task variants from different sources, with new taxonomies and/or new question forms.
Unless otherwise noted, all reported numbers are averaged over three independently fine-tuned runs, and use greedy decoding with temp $=0$.

\noindent\textbf{ID results: our universal models outperform 12 task-specific fine-tuned models.}
As shown in Tab.~\ref{tab:iid-main}:
\textbf{\textsc{UniShot}} achieves the best average accuracy (\textbf{0.759}) and ranks first on \textbf{7/12} tasks; \textbf{\textsc{AgentShots}} performs almost the same as \textsc{UniShot} (0.758 vs.\ 0.759).
 Surprisingly, \textsc{UniShot} and \textsc{AgentShots} both outperform the ensembble performance of the \textbf{12 task-specific Qwen3-VL-8B counterparts (0.740)}, suggesting strong and general positive transfer across SLU tasks.

\noindent\textbf{OOD results: our universal models can generalize well.}
As shown in Tab.~\ref{tab:ood-main},
\textbf{\textsc{AgentShots}} achieves the best average accuracy (\textbf{0.666}).
\textbf{\textsc{UniShot}} also remains competitive, reaching an average accuracy of \textbf{0.617}, which is already substantially higher than all other universal baselines on average (\textbf{best baseline: Gemini-3.0-Flash at 0.581}).
These ID and OOD results suggest that (1) \textsc{UniShot}, an 8B VLM backbone with one LoRA adapter (89.6M trainable parameters) already has enough capacity to memorize the fine-tuning data, and (2) the main advantage of \textsc{AgentShots} as peak specialists is stronger OOD generalization, by capturing dimension-specific info.

\noindent\textbf{Other Observations.} We further observe that (1) Qwen3-VL-8B still falls far behind advanced closed-source VLMs on SLU tasks.
(2) The existing SLU model ShotVL-7B achieves performance close to Gemini-2.5-Flash, but still trails our 8B-scale \textsc{UniShot} and \textsc{AgentShots} by a large margin. (3) Gemini-3.0 versions show clear improvement over Gemini-2.5-Flash, indicating that SLU benefits from progress in general visual understanding. However, Gemini-3.0-Pro, which mainly improves reasoning ability, performs worse than the Gemini-3.0-Flash. This finding suggests that the bottleneck of VLMs for SLU cannot be directly solved by stronger reasoning, which further highlights the significance of our work. App.~\ref{app:backbone-transfer} further verifies that our solutions remain effective when replacing Qwen3-VL-8B with Qwen3.5-VL-9B.

\noindent \textbf{Overall Takeaways.} \textsc{UniShot}: provides strong and balanced one-for-all performance within a dense model. \textsc{AgentShots}: further delivers clearly better OOD generalization performance by sparse-activated LoRA experts. Both solutions are effective and highly adaptable for diverse use cases
\begin{figure}[t]
    \centering
    \includegraphics[width=\linewidth]{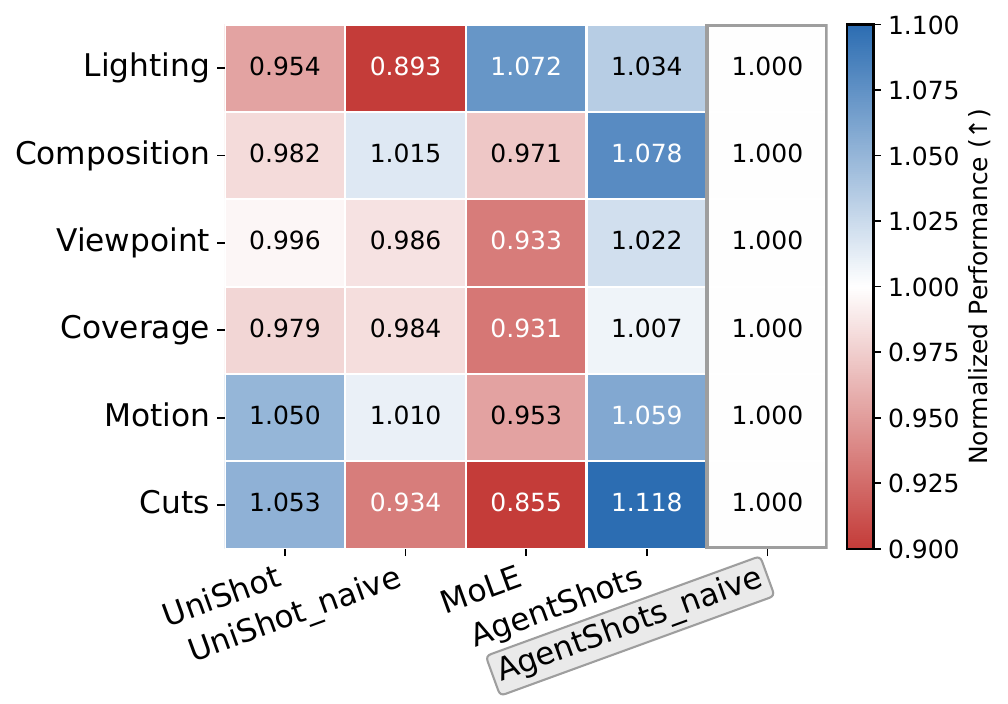}
    \vspace{-4mm}
    \caption{\textbf{Ablation of data mixing strategies.}
    Accuracy is normalized by \textit{AgentShots\_naive}; \(1.0\) means matching specialist trained only on in-dimension data.
    Compared with \textsc{UniShot\_naive}, \textsc{UniShot}  improves the weakest dimensions and average performances, producing a more balanced single-model generalist.
    Compared with \textsc{AgentShots\_naive}, \textsc{AgentShots} improves all six dimensions and achieves the best result on 5 of 6.
    \textit{MoLE}~\cite{wumixture} serves as advanced multi-LoRA baseline, testing the performances of end-to-end data mixing.
    See App.~\ref{app:recipe-ablation} for raw results.}
    \label{fig:strategy-heatmap}
    \vspace{-4mm}
\end{figure}
 \vspace{-5mm}
\subsection{Ablation Studies}
\label{sec:ablation}
 \vspace{-1mm}
Having demonstrated the superiority of \textsc{UniShot} and \textsc{AgentShots}, we next verify the contribution of our data-centric solutions in Sec.~\ref{sec:solutions}.

We keep all other factors (the backbone, LoRA setting, optimizations, etc.) fixed and introduce two ablation variants of our data mixing strategies: \textsc{UniShot\_naive} directly concatenates data from all six dimensions without dynamic balancing. \textit{AgentShots\_naive} trains each expert only on data from its own dimension, without incorporating auxiliary supervision from other dimensions. We also include \textit{MoLE}~\cite{wumixture}, an advanced mixture-of-LoRA training method, with the same expert numbers as \textsc{AgentShots}.

To compare the effectiveness of each method on each dimension, Fig.~\ref{fig:strategy-heatmap} reports accuracy normalized by \textsc{AgentShots\_naive}.
A normalized value of \(1.0\) means that the method is on par with training a specialist using the full data of that dimension only.
We provide the detailed raw results in App.~\ref{app:recipe-ablation}. We have two key observations: \textbf{(1) Dynamic-balanced data mixing improves both average and weakest-dimension performance.}
The average normalized performance rises from 0.970 to 1.000, reaching parity with dimension-specific training.
More importantly, the worst and second-worst dimension scores improve from 0.893/0.934 to 0.954/0.979.
These results show that dynamic-balanced data mixing is important for learning a single model with both strong and
 balanced performance. \textbf{(2) Target-aware data mixing is necessary for strong specialist performance.}
\textsc{AgentShots} improves over \textsc{UniShot\_naive} on every dimension and achieves the best result on 5 of 6 dimensions.
By contrast, although MoLE performs better on \textsc{Lighting}, it falls behind even \textsc{UniShot\_naive} on multiple other dimensions.
This shows that our target-aware data mixing strategy is not only effective but also necessary: strong specialist performance cannot be directly achieved by learned data mixing and more trainable parameters.

\vspace{-1mm}
\section{Further Analysis: How SLU benefits from data heterogeneity and scale?}
\vspace{-1mm}
As discussed in Sec.~\ref{sec:suite}, our constructed SLU-SUITE differs from existing SLU datasets in two key aspects: \emph{data scale}, providing 410K training samples, and \emph{heterogeneity}, aggregating 11 data sources and 22 training tasks with diverse option sets. To understand how these characteristics contribute to general SLU, we conduct further empirical study and theoretical analysis.
\begin{figure}[t]
    \centering
    \includegraphics[width=0.9\linewidth]{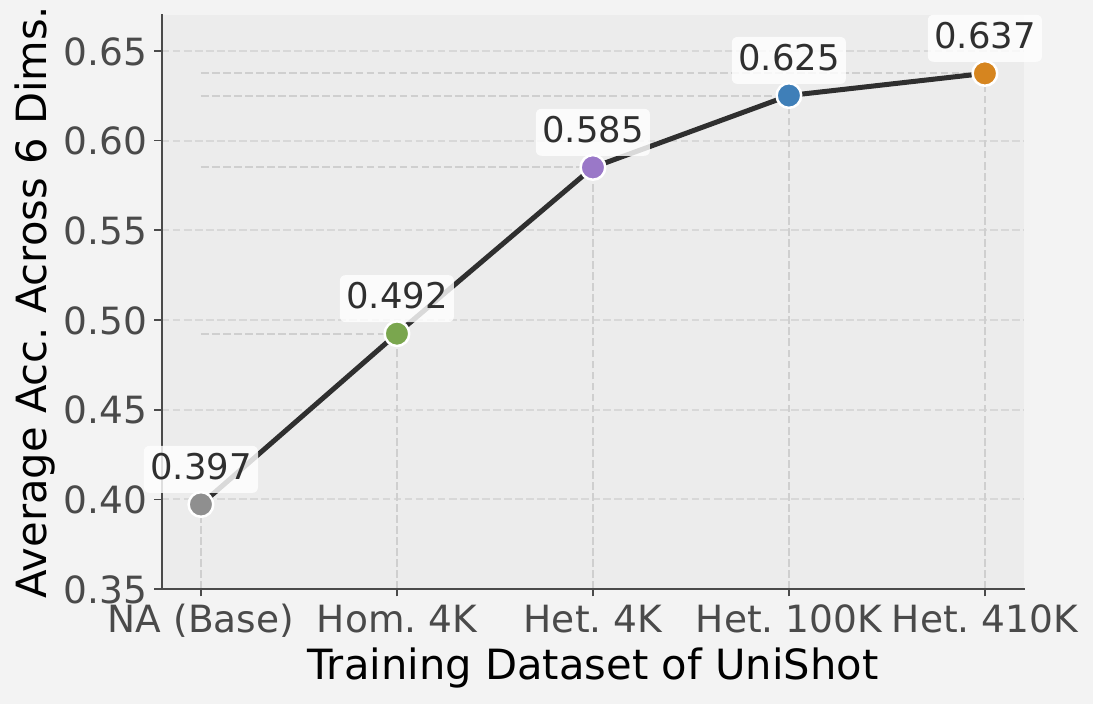}
    \vspace{-2mm}
    \caption{\textbf{Results of dataset variants.} At a matched 4K budget, heterogeneous (multi-source, i.e., randomly sampled from our \textsc{SLU-SUITE}) dataset yields a significant improvement over the homogeneous dataset (mainly ShotBench). Further scaling the heterogeneous pool up to 410K provides progressive performance improvement. See detailed results in App.~\ref{tab:hetero-scale-full}.}
    \label{fig:hetero-scale}
    \vspace{-6mm}
\end{figure}

\noindent \textbf{Empirical Evidence.}
We introduce four dataset variants to train \textsc{UniShot}. To mimic conventional single-source training regime while covering all six SLU dimensions, we first construct \textit{Homogeneous 4K (Hom. 4K)}. We primarily sample from the most comprehensive existing SLU dataset, ShotBench~\cite{shotbench} (the corresponding large-scale training dataset, ShotQA, is not open-sourced). Specifically, we obtain approximately 500 QA pairs for each of its 7 tasks under the \textsc{SLU-SUITE} train-test split, and supplement it with 500 samples for the missing \texttt{Cuts} dimension, yielding roughly 4K samples. Next, we introduce three heterogeneous variants. \textit{Het. 4K} matches the 4K budget but randomly samples from the full \textsc{SLU-SUITE}. Finally, \textit{Het. 100K} and \textit{Het. 410K} (the full training dataset) progressively scale up this heterogeneous pool.  As in Fig.~\ref{fig:hetero-scale}: \textbf{(1) Heterogeneity provides a significant gain at a matched budget}. Replacing the relatively homogeneous 4K control with a heterogeneous 4K mixture improves the average accuracy from 0.492 to 0.585 ($+0.093$). \textbf{(2) Scale progressively strengthens performance once heterogeneity is present}. Scaling the heterogeneous pool from 4K to 100K and then to 410K further improves the average accuracy from 0.585 to 0.625 and 0.637.

\noindent \textbf{Theoretical Analysis.}
To rigorously justify the empirical necessity of data heterogeneity and scale (Sec.~\ref{sec:ablation}), we view SLU annotation through a causal lens (Fig.~\ref{sup:fig:causal_graph_main}). We prove that large-scale multi-source pooling resolves source-induced annotation bias, rendering true semantics identifiable and statistically stable. Full propositions and proofs are deferred to App.~\ref{sup:sec:theory}.
\noindent\textbf{Formulation.} 
Let $H=\phi(X)$ be the semantic representation of a visual input $X$. The observed label $y$ from source $s \in \mathcal{S}$ is governed by a log-linear generation model: $P_s(y \mid H) \propto \exp(f^*(H,y) + b_s(y)) C_s(y)$, where $f^*(H,y)$ is the true \emph{source-invariant semantic compatibility}, $C_s(y) \in \{0,1\}$ is the \emph{structural mask} (defining $s$'s taxonomy), and $b_s(y) \in \mathbb{R}$ is the \emph{preference offset} (subjective annotator bias). 
Under a single-source regime $s_0$, the preference bias is perfectly absorbed into the learned score, i.e., $\tilde{f}(H,y) = f^*(H,y) + b_{s_0}(y)$. The model irrevocably overfits to $s_0$'s subjective bias while $C_{s_0}(y)$ strictly hides out-of-taxonomy labels, making the true semantics $f^*$ unidentifiable.

\vspace{1mm}
\noindent\textbf{Insight 1: Heterogeneity makes bias observable.} 
Breaking this non-identifiability strictly requires \emph{multi-source pooling} (heterogeneity). The mixed training conditional over the pooled data is $P_{\text{tr}}(y \mid H) = \sum_{s \in \mathcal{S}} \alpha_s(H) P_s(y \mid H)$. We prove (Proposition 1) that under mild regularity assumptions, this marginalizes out source-specific preferences, yielding $P_{\text{tr}}(y \mid H) \propto \exp(f^*(H,y) + \kappa(y))$, where $\kappa(y) = \log \sum_s \pi_s C_s(y) \exp(b_s(y) - a_s)$ acts as a composite label prior. Exposing the model to overlapping but non-identical taxonomies effectively neutralizes extreme single-source biases, decoupling $f^*$ from subjective artifacts.

\vspace{1mm}
\noindent\textbf{Insight 2: Scale stabilizes bias estimation.} 
While heterogeneity guarantees identifiability, \emph{large-scale data} is required to stabilize the composite prior $\kappa(y)$. We prove (Proposition 2) that the uniform estimation error of $\kappa$ is bounded by $\mathcal{O}\big(1 / (M_{\min} \sqrt{|\mathcal{S}|})\big)$, where $M_{\min}$ captures the expected label coverage across all sources. Therefore, statistical stability strictly demands both broad source diversity ($|\mathcal{S}| \to \text{large}$) and massive data scale to ensure sufficient label coverage ($M_{\min}$).

\vspace{-2mm}
\section{Conclusion}
\label{sec:conclusion}
\vspace{-1mm}
In this work, we introduce \textsc{SLU-SUITE}, the first general SLU dataset and testbed that enables comprehensive SLU evaluations, with accompanying training datasets at scale. We reveal two key insights: current VLMs are bottlenecked by semantic misalignment more than visual perception; transfer across SLU dimensions is broadly useful but highly uneven and asymmetric. We further propose dedicated data-centric solutions to successfully train two universal SLU models: \textsc{UniShot} and \textsc{AgentShots}, achieving SOTA performance on both in-domain and out-of-domain tasks.

\section*{Acknowledgment}

This work was partly supported by the NSF (Expeditions CCF-1918770, CAREER IIS-2028586, Medium IIS-1955883, Medium IIS-2403240, Medium IIS-2106961), NIH (1R01HL184139), CDC MInD program, Meta, and Dolby faculty gifts.
\bibliography{custom}

\appendix

\clearpage
\appendix

\section*{Appendix}
\section{SLU-SUITE Construction, Split, and QA Unification Details}
\label{app:slu_suite_protocol}

SLU-SUITE is built by unifying 11 public human-labeled or human-verified sources into a common training-and-evaluation suite.
We retain sources only when they contribute meaningful coverage to the six dimensions in Section~\ref{sec:suite}.

During construction, we first collect the original media and annotations for each retained source and discard invalid entries, such as missing files, broken links, corrupted media, or decoding failures.
We then remove repeated media entries before forming the final training and evaluation splits.
After cleaning, we convert all retained annotations into a unified QA interface.

For classification-style tasks, the question explicitly enumerates the candidate options.
Multi-label tasks are converted into choose-all-that-apply QA, binary tasks remain binary QA, and captioning-style annotations are preserved as open-ended supervision.
The captioning-style supervision in SLU-SUITE mainly comes from motion-related tasks, which is why it is used for training only and not for final evaluation.

For evaluation, 11 classification-style task variants from CineTechBench~\cite{cinetechbench} and CameraBench~\cite{camerabench} are fully held out from training and used only for OOD testing.
Among the remaining 22 task variants, we split the 12 variants with more than 3K samples into 80\% training and 20\% ID test partitions, while the other 10 lower-resource variants are used as auxiliary training-only tasks.
From every task variant included in training, we further reserve 20 samples as a small validation split for model selection and dynamic data scheduling.
In addition, any evaluation sample whose media appears in the training pool is removed to reduce cross-source leakage.

We provide the full task-level descriptions in Appendix~\ref{app:slu_suite_task_variants} and the prompt templates used for QA conversion in Appendix~\ref{sec:prompt_templates}.
\section{Detailed Description of SLU-SUITE Task Variants}
\label{app:slu_suite_task_variants}

This section provides task-level definitions for the task variants included in SLU-SUITE.
We follow the same six film-grounded dimensions as in Section~\ref{sec:suite}.
For each task variant, we report its source subset in our suite, its modality when relevant, its option set / label space, and the visual reasoning capability it targets.
We keep semantically related tasks separate whenever their option sets differ, since different option sets correspond to different expert-defined decision boundaries and granularity levels.
Unless otherwise noted, the sample counts reported below refer to the full task partitions available in our V3 dataset engine before the small per-task validation holdout used in the main experiments.

\subsection{Lighting}
Lighting in SLU-SUITE covers source-aware illumination recognition, lighting style identification, attribute-level light characterization, and color-driven mood perception.

\paragraph{SourceCondition (ShotBench; 405 image QA pairs).}
This task is sourced from the ShotBench lighting-type subset (\url{https://vchitect.github.io/ShotBench-project/}) and asks the model to identify the dominant lighting source or environmental lighting condition in a still frame. Candidate answers are dynamically sampled from a broader vocabulary that includes conditions such as daylight, artificial light, fluorescent light, firelight, LED, HMI, moonlight, overcast light, practical light, sunny light, tungsten light, and mixed-light setups. In SLU-SUITE, we use all 405 available QA pairs from this subset. The task therefore emphasizes source-sensitive illumination recognition rather than high-level aesthetic judgment.

\paragraph{Style (CineTechBench; 110 image QA pairs).}
This task is drawn from the CineTechBench lighting subset (\url{https://huggingface.co/datasets/Xinran0906/CineTechBench}) and evaluates whether a model can recognize canonical lighting styles used in cinematography. Depending on the prompt instance, the label space is instantiated as direction-based options (\emph{Back Light}, \emph{Side Light}, \emph{Top Light}), hardness-based options (\emph{Hard Light}, \emph{Soft Light}), or key-based options (\emph{High Key}, \emph{Low Key}). In SLU-SUITE, we use all 110 available QA pairs from this subset. The task targets expert-oriented reasoning about the dominant lighting style rather than merely detecting the light source.

\paragraph{Attribute (ShotBench; 350 image QA pairs).}
This task is sourced from the ShotBench lighting subset (\url{https://vchitect.github.io/ShotBench-project/}) and focuses on fine-grained lighting attributes in still images. Its candidate labels are dynamically composed from an attribute pool including directional properties (\emph{Backlight}, \emph{Side Light}, \emph{Top Light}, \emph{Underlight}, \emph{Edge Light}), contrast-related properties (\emph{High Contrast}, \emph{Low Contrast}), and appearance descriptors such as \emph{Hard Light}, \emph{Soft Light}, and \emph{Silhouette}. In SLU-SUITE, we use all 350 available QA pairs from this subset. Because the options often appear as multi-attribute combinations, this variant captures composite lighting configurations that are common in real cinematographic frames.

\paragraph{ColorPalette (CineTechBench; 60 image QA pairs).}
This task is sourced from the CineTechBench color subset (\url{https://huggingface.co/datasets/Xinran0906/CineTechBench}) and asks the model to classify the dominant cinematic color palette of an image. The label space is a six-way taxonomy consisting of \emph{Red}, \emph{Yellow}, \emph{Blue}, \emph{Green}, \emph{Purple}, and \emph{Black and White}. In SLU-SUITE, we use all 60 available QA pairs from this subset. In our suite, this task complements the lighting tasks by modeling color-driven atmosphere and visual tone.

\subsection{Composition}
The composition dimension focuses on how visual elements are organized within the frame, ranging from explicit composition rules to dominant arrangement patterns and spatial weight distribution.

\paragraph{CompositionRule (KUPCP; 4{,}242 image QA pairs).}
This task is sourced from the KU-PCP composition dataset introduced in \url{https://www.sciencedirect.com/science/article/pii/S1047320318301147} and asks the model to identify which photographic composition rules are present in an image. The label space contains nine rule categories: \emph{Rule of Thirds}, \emph{Vertical}, \emph{Horizontal}, \emph{Diagonal}, \emph{Curved}, \emph{Triangle}, \emph{Center}, \emph{Symmetric}, and \emph{Pattern}. In SLU-SUITE, we use all 4{,}242 available QA pairs from this subset. Since multiple rules may co-occur in a single frame, the task emphasizes structured multi-rule reasoning over frame layout.

\paragraph{CompositionPattern (CineTechBench; 120 image QA pairs).}
This task is drawn from the CineTechBench composition subset (\url{https://huggingface.co/datasets/Xinran0906/CineTechBench}) and focuses on identifying the dominant compositional principle of a frame. The label space is a seven-way taxonomy: \emph{Symmetrical}, \emph{Central}, \emph{Diagonal}, \emph{Rule of Thirds}, \emph{Framing}, \emph{Curved Line}, and \emph{Horizontal}. In SLU-SUITE, we use all 120 available QA pairs from this subset. Compared with CompositionRule, this task emphasizes the primary organizing pattern of the image rather than enumerating all applicable rules.

\paragraph{VisualWeightPlacement (ShotBench; 479 image QA pairs).}
This task is sourced from the ShotBench composition subset (\url{https://vchitect.github.io/ShotBench-project/}) and evaluates how visual weight is distributed across the frame. Candidate answers are dynamically drawn from a broader pool including concepts such as \emph{Balanced}, \emph{Left Heavy}, \emph{Right Heavy}, \emph{Center}, \emph{Symmetrical}, and \emph{Short Side}, with many prompts presenting compound alternatives. In SLU-SUITE, we use all 479 available QA pairs from this subset. This variant therefore targets perceptual reasoning about balance and spatial emphasis instead of explicit rule naming.

\subsection{Viewpoint}
The viewpoint dimension covers camera angle and camera height, including both standard taxonomies and more specialized angle variants used in cinematic analysis.

\paragraph{Angle\_Basic (CineScale2; 9{,}129 image QA pairs).}
This task is sourced from the CineScale2 camera-angle subset (\url{https://cinescale.github.io/camera_al/#dataset}) and evaluates a standard five-way angle taxonomy. The labels are \emph{overhead}, \emph{high}, \emph{neutral}, \emph{low}, and \emph{dutch}. In SLU-SUITE, we use all 9{,}129 available QA pairs from this subset. It serves as the basic viewpoint task in our suite and captures coarse vertical angle changes commonly used in film language.

\paragraph{Angle\_Extended (CineTechBench; 120 image QA pairs).}
This task is drawn from the CineTechBench angle subset (\url{https://huggingface.co/datasets/Xinran0906/CineTechBench}) and extends the standard angle formulation with more specialized cinematic categories. Depending on the prompt instance, the options come from either \{\emph{Diagonal Angle}, \emph{Profile Shot}, \emph{Back Shot}\} or \{\emph{High Angle Shot}, \emph{Low Angle Shot}, \emph{Bird's Eye View}, \emph{Worm's Eye View}\}. In SLU-SUITE, we use all 120 available QA pairs from this subset. The task therefore emphasizes expert-level distinctions that go beyond the basic high-versus-low viewpoint taxonomy.

\paragraph{Angle\_Finegrained (ShotBench; 455 image QA pairs).}
This task is sourced from the ShotBench camera-angle subset (\url{https://vchitect.github.io/ShotBench-project/}) and targets fine-grained viewpoint recognition in still images. The candidate options are dynamically sampled from a larger pool containing \emph{Aerial}, \emph{Overhead}, \emph{High Angle}, \emph{Low Angle}, and \emph{Dutch Angle}, often combined into multi-part alternatives. In SLU-SUITE, we use all 455 available QA pairs from this subset. Relative to Angle\_Basic, this variant places greater emphasis on subtle distinctions and ambiguous boundary cases between nearby angle categories.

\paragraph{Angle\_Cartoon (Cartoon; 8{,}659 image QA pairs).}
This task is sourced from the Cartoon camera-angle subset (\url{https://cinescale.github.io/anime/#dataset}) and applies the same basic viewpoint reasoning to illustrated or animated imagery. The label space contains five options: \emph{dutch}, \emph{high}, \emph{low}, \emph{neutral}, and \emph{overhead}. In SLU-SUITE, we use all 8{,}659 available QA pairs from this subset. Its inclusion broadens the visual domain of viewpoint understanding beyond live-action data.

\paragraph{Height (CineScale2; 15{,}619 image QA pairs).}
This task is drawn from the CineScale2 camera-level subset (\url{https://cinescale.github.io/camera_al/#dataset}) and asks the model to infer camera height from a still frame. The six labels are \emph{aerial\_level}, \emph{eye\_level}, \emph{shoulder\_level}, \emph{hip\_level}, \emph{knee\_level}, and \emph{ground\_level}. In SLU-SUITE, we use all 15{,}619 available QA pairs from this subset. The task focuses on the physical vertical placement of the camera rather than its tilt or angle.

\paragraph{Height\_Cartoon (Cartoon; 3{,}760 image QA pairs).}
This task is sourced from the Cartoon camera-level subset (\url{https://cinescale.github.io/anime/#dataset}) and transfers the same camera-height reasoning problem to illustrated imagery. The label space consists of \emph{aerial}, \emph{eye}, \emph{ground}, \emph{hip}, \emph{knee}, and \emph{shoulder}. In SLU-SUITE, we use all 3{,}760 available QA pairs from this subset. In our suite, it complements Height by testing whether camera-level understanding generalizes across visual styles.

\subsection{Coverage}
The coverage dimension models how much of the subject or scene is included in the frame, spanning classical shot scales, extended scale taxonomies, staging-oriented framing labels, and focal-length cues.

\paragraph{Scale\_Basic (MovieShots; 21{,}614 video QA pairs).}
This task is sourced from the MovieShots shot-size subset released under MovieNet (\url{https://movienet.github.io/}) and uses a five-way shot-scale taxonomy over video clips. The labels are \emph{Long Shot}, \emph{Full Shot}, \emph{Medium Shot}, \emph{Close-Up Shot}, and \emph{Extreme Close-Up Shot}. In SLU-SUITE, we use all 21{,}614 available QA pairs from this subset. It serves as the basic video-based scale recognition task in our suite.

\paragraph{Scale\_Classic (CineTechBench; 140 image QA pairs).}
This task is drawn from the CineTechBench scale subset (\url{https://huggingface.co/datasets/Xinran0906/CineTechBench}) and follows a classical seven-level cinematic shot-scale taxonomy. The labels are \emph{Extreme Close-Up}, \emph{Close-Up}, \emph{Medium Close-Up}, \emph{Medium Shot}, \emph{Medium Long Shot}, \emph{Long Shot}, and \emph{Extreme Long Shot}. In SLU-SUITE, we use all 140 available QA pairs from this subset. Compared with Scale\_Basic, it offers a more canonical and evenly graded coverage hierarchy for expert-oriented shot-scale assessment.

\paragraph{Scale\_Extended (CinemaScale; 19{,}980 video QA pairs).}
This task is sourced from the CinemaScale shot-size subset (\url{https://cinescale.github.io/shotscale/#dataset}) and extends the standard scale taxonomy with additional categories that capture more specialized framing conditions. Its nine labels are \emph{Foreground Shot}, \emph{Extreme Close-Up}, \emph{Close-Up}, \emph{Medium Close-Up}, \emph{Medium Shot}, \emph{Medium Long Shot}, \emph{Long Shot}, \emph{Extreme Long Shot}, and \emph{Insert Shot}. In SLU-SUITE, we use all 19{,}980 available QA pairs from this subset. This formulation covers a broader range of framing outcomes than the basic shot-size taxonomy.

\paragraph{Scale\_Finegrained (ShotBench; 485 image QA pairs).}
This task is drawn from the ShotBench shot-size subset (\url{https://vchitect.github.io/ShotBench-project/}) and targets fine-grained scale reasoning in still images. The candidate answers are dynamically sampled from a neighboring-scale pool including labels such as \emph{Extreme Close Up}, \emph{Close Up}, \emph{Medium Close Up}, \emph{Medium}, \emph{Medium Wide}, \emph{Wide}, and \emph{Extreme Wide}, with many prompts presenting closely related alternatives together. In SLU-SUITE, we use all 485 available QA pairs from this subset. The task is particularly useful for probing difficult decision boundaries between adjacent shot sizes.

\paragraph{Scale\_Cartoon (Cartoon; 5{,}576 image QA pairs).}
This task is sourced from the Cartoon shot-size subset (\url{https://cinescale.github.io/anime/#dataset}) and applies scale classification to illustrated frames. The label space consists of \emph{Close-up-Shot}, \emph{Medium-Shot}, and \emph{Long-Shot}. In SLU-SUITE, we use all 5{,}576 available QA pairs from this subset. This variant introduces a stylized visual domain while preserving a compact scale taxonomy.

\paragraph{Scale\_Historical (HistShotScale; 1{,}885 image QA pairs).}
This task is drawn from the HistShotScale subset (\url{https://github.com/dahe-cvl/ICPRAM2022_histshotV1}) and focuses on shot-size recognition in historical film imagery. The six labels are \emph{Close-Up}, \emph{Extreme-Long-Shot}, \emph{Intertitle}, \emph{Long-Shot}, \emph{Medium-Shot}, and \emph{Not Available}. In SLU-SUITE, we use all 1{,}885 available QA pairs from this subset. The presence of non-standard categories such as \emph{Intertitle} and \emph{Not Available} makes this subset distinct from conventional shot-scale benchmarks.

\paragraph{Staging (FilmShots; 1{,}508 image QA pairs).}
This task is sourced from the FilmShots shot-type subset, which in our local curation is derived from stills originally collected from Shot Cafe (\url{https://shot.cafe/}). The six labels are \emph{clean single}, \emph{close-up}, \emph{extreme wide}, \emph{medium shot}, \emph{profile shot}, and \emph{wide shot}. In SLU-SUITE, we use all 1{,}508 available QA pairs from this subset. Unlike pure scale tasks, it blends subject count, orientation, and staging cues into a compact framing taxonomy.

\paragraph{StagingScaleMix (ShotBench; 445 image QA pairs).}
This task is drawn from the ShotBench shot-framing subset (\url{https://vchitect.github.io/ShotBench-project/}) and models staging-oriented framing configurations. Candidate answers are dynamically composed from labels such as \emph{Insert}, \emph{Single}, \emph{Over the Shoulder}, \emph{2 Shot}, \emph{3 Shot}, \emph{Group Shot}, and \emph{Establishing Shot}, often appearing as compound alternatives. In SLU-SUITE, we use all 445 available QA pairs from this subset. The task therefore mixes subject arrangement and framing scale into a unified decision space.

\paragraph{FocalLength (CineTechBench; 60 image QA pairs).}
This task is sourced from the CineTechBench focal-length subset (\url{https://huggingface.co/datasets/Xinran0906/CineTechBench}) and asks the model to infer the lens category used for image capture. The label space contains \emph{Standard Lens}, \emph{Medium Focal Lens}, \emph{Telephoto Lens}, \emph{Fisheye Lens}, and \emph{Macro Lens}. In SLU-SUITE, we use all 60 available QA pairs from this subset. In our suite, this task complements shot-scale reasoning with a lens-based account of framing geometry.

\subsection{Motion}
The motion dimension covers both direct camera-movement classification and question-answering formulations over motion presence, complexity, direction, shaking, speed, and free-form motion description.

\paragraph{Move\_Coarse (MovieShots; 21{,}614 video QA pairs).}
This task is sourced from the MovieShots camera-movement subset released under MovieNet (\url{https://movienet.github.io/}) and uses a coarse five-way motion taxonomy over video clips. The labels are \emph{Static Shot}, \emph{Motion Shot}, \emph{Push Shot}, \emph{Pull Shot}, and \emph{Multi\_movement}. In SLU-SUITE, we use all 21{,}614 available QA pairs from this subset. It provides a compact entry point for motion recognition before moving to more compositional formulations.

\paragraph{Move\_Compound\_A (ShotBench; 464 video QA pairs).}
This task is drawn from the ShotBench camera-movement subset (\url{https://vchitect.github.io/ShotBench-project/}) and focuses on compound or fine-grained motion labels. Candidate options are dynamically sampled from a broad pool containing actions such as \emph{Pan Left}, \emph{Pan Right}, \emph{Tilt Up}, \emph{Tilt Down}, \emph{Zoom In}, \emph{Pull Out}, \emph{Arc}, \emph{Static}, and many two-stage compositions such as ``pan left, then tilt down.'' In SLU-SUITE, we use all 464 available QA pairs from this subset. This subset is designed to probe difficult compositional motion reasoning rather than coarse motion detection alone.

\paragraph{Move\_Compound\_B (CineTechBench; 93 video QA pairs).}
This task is sourced from the CineTechBench camera-movement subset (\url{https://huggingface.co/datasets/Xinran0906/CineTechBench}) and evaluates expert-oriented motion categories over short clips. Its option space includes both single movements and compositional patterns, with examples such as \emph{Fixed Shot}, \emph{Crane Shot}, \emph{Dolly In Shot}, \emph{Dolly Out Shot}, \emph{Rolling Clockwise Shot}, and combined labels such as \emph{Pan Left Shot with Zoom Out Shot}. In SLU-SUITE, we use all 93 available QA pairs from this subset. Compared with Move\_Compound\_A, this variant uses a more formal cinematic naming style.

\paragraph{Move\_Captioning (CameraBench; 35{,}050 video QA pairs).}
This task is drawn from the CameraBench motion-caption subset (\url{https://linzhiqiu.github.io/papers/camerabench/}) and is the only open-ended motion variant in our suite. Instead of selecting from a predefined option list, the model is asked to summarize the camera movement in natural language through prompts such as ``How does the camera move in this video?'' or ``Summarize the camera movement in this video.'' In SLU-SUITE, we use all 35{,}050 available QA pairs from this subset. The task targets descriptive motion understanding rather than closed-set classification.

\paragraph{VQA\_Complexity (CameraBench; 4{,}284 video QA pairs).}
This task is sourced from the CameraBench motion-complexity VQA subset (\url{https://linzhiqiu.github.io/papers/camerabench/}) and formulates motion understanding as binary yes/no questions. The questions query whether the clip exhibits \emph{no camera motion}, \emph{minor motion}, \emph{simple motion}, or \emph{complex motion}. In SLU-SUITE, we use all 4{,}284 available QA pairs from this subset. It isolates complexity perception as a separate reasoning capability within motion analysis.

\paragraph{VQA\_Movement (CameraBench; 17{,}459 video QA pairs).}
This task is drawn from the CameraBench movement VQA subset (\url{https://linzhiqiu.github.io/papers/camerabench/}) and asks yes/no questions about concrete motion primitives. Typical prompts query actions such as panning left or right, tilting up or down, trucking, dollying, or zooming. In SLU-SUITE, we use all 17{,}459 available QA pairs from this subset. The task emphasizes recognition of specific motion directions and operators in video clips.

\paragraph{VQA\_Shaking (CameraBench; 3{,}896 video QA pairs).}
This task is sourced from the CameraBench shaking VQA subset (\url{https://linzhiqiu.github.io/papers/camerabench/}) and evaluates whether the clip exhibits different levels of camera instability. The underlying labels include states such as \emph{no shaking}, \emph{minimal shaking}, \emph{unsteady}, and \emph{very unsteady}, each queried through binary yes/no prompts. In SLU-SUITE, we use all 3{,}896 available QA pairs from this subset. This variant specifically targets stability perception rather than intentional camera movement.

\paragraph{VQA\_Speed (CameraBench; 3{,}213 video QA pairs).}
This task is drawn from the CameraBench speed VQA subset (\url{https://linzhiqiu.github.io/papers/camerabench/}) and asks whether the observed motion speed is \emph{slow}, \emph{regular}, or \emph{fast}. Although implemented through yes/no questions, the task functions as a speed-sensitive motion taxonomy over short clips. In SLU-SUITE, we use all 3{,}213 available QA pairs from this subset. It complements the movement and complexity tasks by isolating temporal intensity.

\paragraph{VQA\_Presence (CameraBench; 157{,}552 video QA pairs).}
This task is sourced from the CameraBench imbalanced raw VQA subset (\url{https://linzhiqiu.github.io/papers/camerabench/}) and asks whether a specific motion description is present in the clip. The queried descriptions cover both atomic and compositional patterns, including examples such as craning downward in an arc, tracking from a specific viewpoint, or moving forward through the scene. In SLU-SUITE, we use all 157{,}552 available QA pairs from this subset. Because it contains a large and diverse set of motion predicates, this subset provides broad coverage of motion-presence recognition.

\paragraph{VQA\_MixedType (CameraBench; 38{,}672 video QA pairs).}
This task is drawn from the CameraBench balanced VQA subset (\url{https://linzhiqiu.github.io/papers/camerabench/}) and mixes multiple motion attributes within a unified yes/no format. Its prompts cover presence and absence, motion speed, motion complexity, stability, and descriptive matching cases, for example asking whether the camera is completely still or whether a clip matches a given motion description. In SLU-SUITE, we use all 38{,}672 available QA pairs from this subset. In our suite, this variant complements VQA\_Presence by offering a more balanced and heterogeneous motion-question distribution.

\subsection{Cuts}
The cuts dimension models transitions between shots and focuses on the semantics of inter-shot editing decisions in video.

\paragraph{InterShotCutType (MovieCuts; 109{,}827 video QA pairs).}
This task is sourced from the MovieCuts cut-type subset (\url{https://github.com/PardoAlejo/MovieCuts}) and asks the model to identify the type of cut used in a video example. The label space contains eight categories: \emph{cut-on-action}, \emph{cut-away}, \emph{cross-cut}, \emph{emphasis/deemphasis}, \emph{match-cut}, \emph{smash-cut}, \emph{reaction-in/reaction-out-cut}, and \emph{speaker-change}. In SLU-SUITE, we use all 109{,}827 available QA pairs from this subset. This large-scale subset provides dedicated supervision for understanding inter-shot editing grammar.

\section{Question and Answer Prompt Templates}
\label{sec:prompt_templates}

We convert heterogeneous SLU annotations into unified QA pairs using task-specific prompt templates.
Figure~\ref{fig:slu_prompt_mcq} shows the main multiple-choice template.
In every multimodal prompt, the visual placeholder token is placed on the first line and is instantiated as either \texttt{<image>} or \texttt{<video>} according to the source modality.
Most task variants use multiple-choice templates with explicit answer-format instructions; multi-label tasks use choose-all-that-apply templates; CameraBench contributes additional binary VQA templates; and motion-related captioning tasks use open-ended templates that are used only during training.

\begin{figure*}[h]
    \centering
    \begin{tcolorbox}
    \begin{lstlisting}[breaklines=true]
<VISUAL_TOKEN>
You are a cinematography expert.
Question: {task-specific question}
Options:
A. {option_1}
B. {option_2}
C. {option_3}
D. {option_4}
...
If multiple options are correct, connect them with ' && '.
Respond with 'Letter. OptionText' (for multiple: 'A. ... && B. ...').
    \end{lstlisting}
    \end{tcolorbox}
    \caption{Unified multiple-choice prompt template used for the majority of SLU-SUITE tasks, including composition, coverage, viewpoint, lighting, motion classification, and cut-type recognition. The multimodal placeholder is always placed at the first line.}
    \label{fig:slu_prompt_mcq}
\end{figure*}

\section{Detailed Data Partitions: Train, ID, and OOD}
\label{app:data_splits}

The exact task assignments are detailed below:

\begin{itemize}[leftmargin=*, topsep=4pt, itemsep=4pt]
    \item \textbf{OOD Evaluation Partition (11 tasks):} These tasks are completely held out during training and used strictly for OOD testing to evaluate zero-shot or generalization capabilities.
    \begin{itemize}[noitemsep, topsep=2pt]
        \item CompositionPattern
        \item ScaleClassic
        \item FocalLength
        \item AngleExtended
        \item MoveCompound
        \item VQAMovement
        \item VQAComplexity
        \item VQAShaking
        \item VQASpeed
        \item LightingStyle
        \item ColorPalette
    \end{itemize}

    \item \textbf{ID Training \& Evaluation Partition (12 tasks):} These tasks have large data volumes ($>3$K samples). We randomly split each task into 80\% for training and 20\% for ID evaluation.
    \begin{itemize}[noitemsep, topsep=2pt]
        \item CompositionRule
        \item ScaleExtended
        \item ScaleHistorical
        \item ScaleCartoon
        \item ScaleBasic
        \item Staging
        \item AngleBasic
        \item AngleCartoon
        \item Height
        \item HeightCartoon
        \item MoveCoarse
        \item InterShotCutType
    \end{itemize}

    \item \textbf{Training-Only Partition (10 tasks):} All remaining tasks are injected 100\% into the training set to enrich the model's semantic representations and task diversity. They are not explicitly evaluated in the ID test set.
    \begin{itemize}[noitemsep, topsep=2pt]
        \item VisualWeightPlacement
        \item Scale\_Finegrained
        \item StagingScaleMix
        \item Angle\_Finegrained
        \item Move\_Compound\_A
        \item Move\_Captioning
        \item VQA\_Presence
        \item VQA\_MixedType
        \item SourceCondition
        \item Attribute
    \end{itemize}
\end{itemize}


\subsection{Transfer-In and Transfer-Out Summary}
\label{app:transfer-summary}

To complement Fig.~\ref{fig:transfer_matrix}, Table~\ref{tab:transfer-summary} summarizes transfer statistics from two perspectives:
\textbf{transfer-in receptiveness}, which measures how often a target benefits from external sources, and
\textbf{transfer-out helpfulness}, which measures how broadly a source dimension improves other targets.
All gains are computed relative to the \textbf{Base} model.

\begin{table*}[t]
    \centering
    \small
    \setlength{\tabcolsep}{5pt}
    \caption{\textbf{Summary of cross-dimensional transfer utility.}
    ``Incoming'' measures how a target benefits from external source dimensions.
    ``Outgoing'' measures how much a source helps other targets.
    Mean gains are averaged over all five off-diagonal targets relative to the Base model.}
    \label{tab:transfer-summary}
    \begin{tabular}{lcccc}
        \toprule
        Dimension & Incoming + / 5 & Mean Incoming Gain & Outgoing + / 5 & Mean Outgoing Gain \\
        \midrule
        Lighting    & 5/5 & +0.122 & 3/5 & +0.078 \\
        Composition & 5/5 & +0.299 & 4/5 & +0.044 \\
        Viewpoint   & 5/5 & +0.032 & 5/5 & +0.108 \\
        Coverage    & 1/5 & -0.010 & 5/5 & +0.115 \\
        Motion      & 5/5 & +0.056 & 4/5 & +0.094 \\
        Cuts        & 4/5 & +0.008 & 4/5 & +0.069 \\
        \bottomrule
    \end{tabular}
\end{table*}

The table highlights the asymmetry already visible in Fig.~\ref{fig:transfer_matrix}.
For example, \textsc{Coverage} is one of the strongest sources in terms of transfer-out helpfulness, but one of the weakest targets in terms of transfer-in receptiveness.
This confirms that cross-dimensional utility in SLU is not symmetric and motivates target-aware data construction.

\newpage

\section{RQ1 Experimental Details}
\label{app:rq1_details}

This appendix provides the experimental details omitted from Sec.~\ref{sec:rq1_bottleneck}.

\paragraph{Controlled setup.}
To isolate module choice from cross-dimensional transfer, each SLU dimension is trained and evaluated independently.
All adapter-placement settings use the same Qwen3-VL-8B backbone, the same dimension-specific train/test split, the same LoRA rank (\(r=32\)), and the same optimizer, batch size, and training schedule.
The only difference is which module receives LoRA updates: \textbf{Connector}, \textbf{Vision}, \textbf{LM}, \textbf{LM+Connector}, or \textbf{All}.

\paragraph{Why \textsc{Viewpoint}-\textbf{Vision} is omitted.}
Vision-only adaptation on \textsc{Viewpoint} frequently produced degenerate generations and could not complete evaluation reliably.
We therefore omit this cell from Table~\ref{tab:ablation-lora-dim} and report both ``Avg. (all 6)'' and ``Avg. (shared 5)'' in the main paper.

\paragraph{Parameter counts.}
The trainable parameter counts in Table~\ref{tab:ablation-lora-dim} correspond to LoRA-updated parameters only, not the full frozen backbone.
This is why \textbf{LM+Connector} is more parameter-efficient than \textbf{All} while still achieving better average performance.

\paragraph{Interpretation.}
The main takeaway of RQ1 is not that visual information is unimportant for SLU, but that under controlled PEFT with a strong pretrained VLM, the strongest \emph{recoverable} gains concentrate on the language model and connector.
This motivates \textbf{LM+Connector} as the shared backbone for all later analyses and models.

\section{Transfer Summary for RQ2}
\label{app:transfer_details}

Table~\ref{tab:transfer-summary} summarizes the transfer statistics underlying Fig.~\ref{fig:transfer_matrix}.
For each target dimension, we report how many off-diagonal sources outperform the pretrained baseline (\textbf{transfer-in positive count}) and the mean off-diagonal gain over the baseline.
For each source dimension, we analogously report \textbf{transfer-out positive count} and mean gain across all other targets.

\begin{table*}[t]
\centering
\small
\setlength{\tabcolsep}{5pt}
\caption{\textbf{Transfer-in and transfer-out summary derived from Fig.~\ref{fig:transfer_matrix}.}
Positive count reports how many off-diagonal transfers improve over the pretrained baseline.
Mean gain is computed as the average off-diagonal improvement over the baseline under the natural SLU-SUITE regime.}
\label{tab:transfer-summary}
\resizebox{0.8\textwidth}{!}{%
\begin{tabular}{lcc|lcc}
\toprule
\multicolumn{3}{c|}{\textbf{Target-side transfer-in}} & \multicolumn{3}{c}{\textbf{Source-side transfer-out}} \\
\textbf{Target} & \textbf{Positive Count} & \textbf{Mean Gain} & \textbf{Source} & \textbf{Positive Count} & \textbf{Mean Gain} \\
\midrule
Lighting    & 5 / 5 & +0.122 & Lighting    & 3 / 5 & +0.078 \\
Composition & 5 / 5 & +0.299 & Composition & 4 / 5 & +0.044 \\
Viewpoint   & 5 / 5 & +0.032 & Viewpoint   & 5 / 5 & +0.108 \\
Coverage    & 1 / 5 & -0.010 & Coverage    & 5 / 5 & +0.115 \\
Motion      & 5 / 5 & +0.056 & Motion      & 4 / 5 & +0.094 \\
Cuts        & 4 / 5 & +0.008 & Cuts        & 4 / 5 & +0.069 \\
\bottomrule
\end{tabular}}
\vspace{-3mm}
\end{table*}

\paragraph{Discussion.}
The table quantifies three patterns already visible in Fig.~\ref{fig:transfer_matrix}.
First, transfer is broadly positive: \(25\) of the \(30\) off-diagonal source--target pairs outperform the pretrained baseline.
Second, transfer-in utility is highly target-dependent: \textsc{Coverage} is difficult to help, while \textsc{Composition} and \textsc{Lighting} benefit strongly from auxiliary sources.
Third, transfer is asymmetric: \textsc{Coverage} is the strongest donor by mean transfer-out gain but the weakest recipient by mean transfer-in gain.

\section{UniShot Implementation Details}
\label{app:unishot_details}

This appendix expands Sec.~\ref{sec:unishot} with implementation details that are omitted from the main paper for space.

\paragraph{Dimension-level reweighting.}
UniShot reweights training at the \emph{dimension level}.
At each sampling step, it first samples one source dimension according to the current distribution \(p^{(t)}\), and then draws training instances from that dimension.
This design directly matches the high-level transfer matrix in Fig.~\ref{fig:transfer_matrix}.

\paragraph{Validation construction.}
Every training task variant contributes a small held-out validation set.
In our implementation, we reserve 20 samples from each training task variant.
The dimension-level validation score \(V_j^{(t)}\) is computed as the macro-average over all task variants belonging to dimension \(j\).
These validation samples are used only to estimate a coarse under-optimization signal for reweighting, not for final model reporting.

\paragraph{Target-only anchors.}
The target-only anchor \(A_{j\rightarrow j}\) is the performance obtained by training only on dimension \(j\) and evaluating on \(j\).
It serves as a practical reference for how far the current balanced model still lags behind a dimension-specialized solution.

\paragraph{Hyperparameters.}
UniShot uses reweighting interval \(K\), exploration coefficient \(\lambda\), and smoothing coefficient \(\beta\) as defined in Sec.~\ref{sec:unishot}.
Please fill in the final numerical values here according to the training run used in the paper.

\begin{algorithm}[t]
\caption{Dynamic Dimension Mixing for UniShot}
\label{alg:dynamic_mixing_app}
\small
\begin{algorithmic}[1]
\Require Dimension set \(\mathcal{D}\); transfer prior \(\tilde{T}_{i\rightarrow j}\); target-only anchors \(A_{j\rightarrow j}\); reweighting interval \(K\); hyperparameters \(\lambda,\beta\)
\Ensure UniShot model \(f\); final sampling distribution \(p\)

\State \textbf{Initialize:} \(p_i^{(0)} \gets 1/|\mathcal{D}|\), for all \(i \in \mathcal{D}\)

\For{reweighting round \(t=0,1,\dots,T-1\)}
    \State Train the current model for \(K\) steps using \(p^{(t)}\)
    \State Compute dimension-level validation scores \(V_j^{(t)}\), for all \(j \in \mathcal{D}\)
    \State \(g_j^{(t)} \gets \left[1-\frac{V_j^{(t)}}{A_{j\rightarrow j}+\epsilon}\right]_+\)
    \State \(u_i^{(t)} \gets \sum_{j\in\mathcal{D}} \tilde{T}_{i\rightarrow j}\, g_j^{(t)}\)
    \State \(\hat{p}_i^{(t+1)} \gets (1-\lambda)\frac{1}{|\mathcal{D}|} + \lambda \frac{u_i^{(t)}}{\sum_{k\in\mathcal{D}} u_k^{(t)}+\epsilon}\)
    \State \(p_i^{(t+1)} \gets (1-\beta)p_i^{(t)} + \beta \hat{p}_i^{(t+1)}\)
\EndFor

\State \Return final model \(f\) and distribution \(p\)
\end{algorithmic}
\end{algorithm}

\paragraph{Interpretation.}
UniShot uses the transfer prior \(\tilde{T}_{i\rightarrow j}\) to estimate which source dimensions are likely to help which targets, and combines it with the current validation gaps \(g_j^{(t)}\) to decide where more training probability mass should go.
In this sense, UniShot is neither purely static nor purely reactive: it starts from a globally estimated transfer structure but updates it according to the current under-optimized dimensions.

\section{Routing Details for AgentShots}
\label{app:agentshots-routing}

AgentShots uses \textbf{prompt-based routing} rather than oracle task IDs or a separately trained routing model.
At inference time, the shared frozen backbone reads the input task prompt together with short descriptions of the six experts and selects one expert for top-1 activation. Because the routing task relies on explicit keyword and semantic matching between the prompt and the six high-level dimension descriptions, the 8B backbone achieves near-perfect routing accuracy, making an oracle ablation unnecessary.

\paragraph{Expert descriptions.}
Each expert corresponds to one high-level SLU dimension:
\textbf{Lighting}, \textbf{Composition}, \textbf{Viewpoint}, \textbf{Coverage}, \textbf{Motion}, and \textbf{Cuts}.
The routing prompt provides a short description of the scope of each expert, e.g., whether the query concerns lighting style, composition pattern, camera angle, shot distance, motion behavior, or edit transition.

\paragraph{Routing prompt template.}
A simplified template is shown below:
\begin{quote}
\small
You are given a visual question about shot language understanding.
Choose the single most suitable expert from the following six options:
(1) Lighting, (2) Composition, (3) Viewpoint, (4) Coverage, (5) Motion, (6) Cuts.
Return only the expert name that best matches the question.
\end{quote}

\paragraph{Why prompt-based routing?}
Our goal is to keep AgentShots lightweight and compatible with the shared backbone.
Instead of introducing a separately trained router, we reuse the backbone's language understanding ability to map the task prompt to the most relevant expert.
This preserves the same active inference-time parameter path as UniShot while enabling expert specialization.

\paragraph{Recommended qualitative examples.}
In the final appendix, it is useful to include two to four qualitative routing examples showing
(i) the input task prompt,
(ii) the selected expert, and
(iii) the final prediction.
Such examples help clarify that routing is driven by prompt semantics rather than dataset-specific IDs.
\newpage

\section{Additional Ablation Results}
\label{app:ablation}

This appendix provides the exact numbers for the ablations summarized in Sec.~\ref{sec:ablation}.
We first report the full recipe-ablation results, including a comparison with \textbf{MoLE}, a generic mixture-of-LoRA baseline.
We then give the full dimension-wise breakdown for the data heterogeneity and scale study, followed by a transfer check on a newer Qwen3.5-VL backbone.

\subsection{Recipe Ablation Details}
\label{app:recipe-ablation}

Tables~\ref{tab:strategy-abs} and~\ref{tab:strategy-norm} provide the full results behind Fig.~\ref{fig:strategy-heatmap}.
Besides the two naive variants, we also compare against \textbf{MoLE}~\cite{wumixture}, which uses the same six experts and top-1 activation budget as AgentShots, but learns routing end-to-end.
The results show that generic multi-LoRA routing alone is not sufficient:
MoLE helps certain dimensions such as lighting, but fails to recover the consistent cross-dimension gains of AgentShots and remains below both UniShot and AgentShots on average.
This again highlights the importance of transfer-aware data construction rather than relying on automatic routing alone.

\begin{table*}[h]
    \centering
    \small
    \setlength{\tabcolsep}{7pt}
    \caption{\textbf{Recipe ablation (absolute scores): AgentShots is best on average, while UniShot improves balance over naive mixing.}
    \textbf{UniShot-naive} directly mixes all six dimensions without dynamic balancing.
    \textbf{AgentShots-naive} trains each expert only on its own target-dimension data.
    \textbf{MoLE} uses the same six experts and top-1 activation budget as AgentShots, but learns routing end-to-end.
    Dynamic balancing mainly improves the weakest dimensions of UniShot, while target-aware mixtures consistently improve AgentShots over isolated expert training.}
    \label{tab:strategy-abs}
    \begin{tabular}{lccccc}
        \toprule
        Dimension & UniShot & UniShot-naive & AgentShots & AgentShots-naive & MoLE \\
        \midrule
        Composition & 0.6117 & 0.6325 & 0.6717 & 0.6230 & 0.6050 \\
        Coverage    & 0.6530 & 0.6560 & 0.6718 & 0.6670 & 0.6211 \\
        Cuts        & 0.4000 & 0.3550 & 0.4250 & 0.3800 & 0.3250 \\
        Lighting    & 0.6303 & 0.5902 & 0.6833 & 0.6610 & 0.7083 \\
        Motion      & 0.7599 & 0.7309 & 0.7667 & 0.7240 & 0.6899 \\
        Viewpoint   & 0.7767 & 0.7690 & 0.7970 & 0.7800 & 0.7277 \\
        \midrule
        Avg.        & 0.6386 & 0.6223 & \textbf{0.6693} & 0.6392 & 0.6128 \\
        \bottomrule
    \end{tabular}
\end{table*}

\begin{table*}[h]
    \centering
    \small
    \setlength{\tabcolsep}{6pt}
    \caption{\textbf{Recipe ablation (normalized scores): UniShot improves the floor, while AgentShots improves every dimension.}
    Results are normalized by \textbf{AgentShots-naive} (\(=1.0\)).
    Values above \(1.0\) indicate improvement over the naive routed-specialist baseline.
    We additionally report the average, worst-case, and second worst-case performance for each method across the six dimensions.
    For \textbf{AgentShots-naive}, all entries are \(1.0\) by definition.}
    \label{tab:strategy-norm}
    \begin{tabular}{lccccc}
        \toprule
        Dimension & UniShot & UniShot-naive & AgentShots & AgentShots-naive & MoLE \\
        \midrule
        Lighting    & 0.9536 & 0.8929 & 1.0337 & 1.0000 & 1.0716 \\
        Composition & 0.9819 & 1.0152 & 1.0782 & 1.0000 & 0.9711 \\
        Viewpoint   & 0.9958 & 0.9859 & 1.0218 & 1.0000 & 0.9329 \\
        Coverage    & 0.9790 & 0.9835 & 1.0072 & 1.0000 & 0.9312 \\
        Motion      & 1.0496 & 1.0095 & 1.0590 & 1.0000 & 0.9529 \\
        Cuts        & 1.0526 & 0.9342 & 1.1184 & 1.0000 & 0.8553 \\
        \midrule
        Avg.        & 1.0021 & 0.9702 & \textbf{1.0530} & 1.0000 & 0.9525 \\
        Worst       & 0.9536 (\textit{Lighting}) & 0.8929 (\textit{Lighting}) & 1.0072 (\textit{Coverage}) & 1.0000 & 0.8553 (\textit{Cuts}) \\
        2nd Worst   & 0.9790 (\textit{Coverage}) & 0.9342 (\textit{Cuts}) & 1.0218 (\textit{Viewpoint}) & 1.0000 & 0.9312 (\textit{Coverage}) \\
        \bottomrule
    \end{tabular}
\end{table*}

\subsection{Controlled Study on Data Heterogeneity and Scale}
\label{app:hetero-scale}

Table~\ref{tab:hetero-scale-full} reports the full per-dimension numbers for the controlled study in Sec.~\ref{sec:ablation}.
At matched 4K size, heterogeneous supervision already brings large gains over homogeneous data, especially for composition and viewpoint.
As the heterogeneous pool grows further, motion and cuts continue to improve, while some dimensions such as coverage and viewpoint start to saturate.
This explains why the average continues to rise with scale, but with smaller marginal gains.

\subsection{Transfer to a Newer Backbone}
\label{app:backbone-transfer}

Finally, we replace the original Qwen3-VL-8B backbone with the newer \textbf{Qwen3.5-VL-9B}, while keeping the UniShot training setting unchanged.
Table~\ref{tab:qwen35_backbone_transfer} shows that SLU-SUITE supervision still brings large gains across all six dimensions.
This result suggests that the value of our training data and task construction is not tied to a single backbone family.

\begin{table}[h]
    \centering
    \small
    \setlength{\tabcolsep}{7pt}
    \caption{\textbf{Transfer to a newer backbone: SLU-SUITE supervision remains effective on Qwen3.5-VL-9B.}
    We replace Qwen3-VL-8B with \textbf{Qwen3.5-VL-9B} and keep the UniShot training setting unchanged.
    SLU-SUITE supervision consistently improves over the raw model across all six shot-language dimensions.}
    \label{tab:qwen35_backbone_transfer}
    \begin{tabular}{lcc}
        \toprule
        Dimension & Raw Qwen3.5 & UniShot on SLU-SUITE \\
        \midrule
        Composition & 0.0617 & 0.6100 \\
        Coverage    & 0.4569 & 0.6662 \\
        Cuts        & 0.2300 & 0.3350 \\
        Lighting    & 0.5326 & 0.7076 \\
        Motion      & 0.6018 & 0.7407 \\
        Viewpoint   & 0.5000 & 0.8027 \\
        \midrule
        Avg.        & 0.3972 & 0.6437 \\
        \bottomrule
    \end{tabular}
\end{table}

\begin{table*}[h]
    \centering
    \small
    \setlength{\tabcolsep}{6pt}
    \caption{\textbf{Data heterogeneity and scale (full breakdown): heterogeneity gives the main jump, while scale adds uneven gains.}
    \textbf{Raw} is the pretrained model without SLU fine-tuning.
    \textbf{Homogeneous 4K} and \textbf{Heterogeneous 4K} use the same training size but differ in source diversity.
    \textbf{Heterogeneous 100K} and \textbf{Heterogeneous 410K} further scale the same heterogeneous pool.
    Note that the lighting split is the same in the homogeneous and heterogeneous 4K settings, so no clear gain is expected for lighting at that scale.}
    \label{tab:hetero-scale-full}
    \begin{tabular}{lccccc}
        \toprule
        Dimension & Raw & Homogeneous 4K & Heterogeneous 4K & Heterogeneous 100K & Heterogeneous 410K \\
        \midrule
        Composition & 0.0617 & 0.3775 & 0.5858 & 0.6217 & 0.6167 \\
        Coverage    & 0.4569 & 0.5178 & 0.5878 & 0.6573 & 0.6407 \\
        Cuts        & 0.2300 & 0.2800 & 0.3200 & 0.3550 & 0.4000 \\
        Lighting    & 0.5326 & 0.6439 & 0.6432 & 0.5939 & 0.6303 \\
        Motion      & 0.6018 & 0.6293 & 0.6882 & 0.7300 & 0.7617 \\
        Viewpoint   & 0.5000 & 0.5057 & 0.6853 & 0.7927 & 0.7753 \\
        \midrule
        Avg.        & 0.3972 & 0.4924 & 0.5851 & 0.6251 & 0.6375 \\
        \bottomrule
    \end{tabular}
\end{table*}

\newpage

\section{Detailed Theoretical Analysis}
\label{sup:sec:theory}

We study two questions in sequence.
First, \emph{why is large-scale multi-source data necessary} for subjective supervision in SLU?
Second, \emph{why is a mixed pointwise + pairwise objective necessary} even after such data is available?
Our goal is not to claim complete debiasing.
Instead, we formalize a narrower but more defensible statement:
\emph{large-scale multi-source data makes source-induced annotation bias observable and statistically stable, while pairwise preference locally contracts the label-only priors that pointwise training tends to absorb.}
\subsection{Formulation: Source-Dependent Annotation Mechanisms}
\label{sup:subsec:theory_formulation}

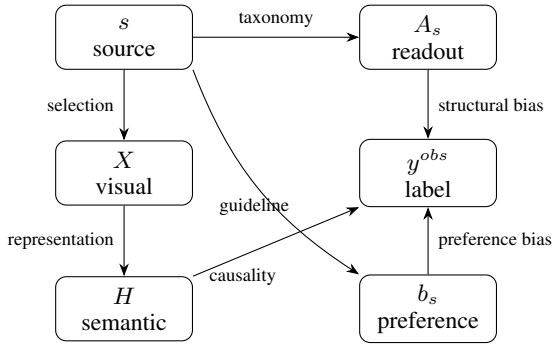
\begin{figure}[t]
\centering
\begin{tikzpicture}[
    x=1cm, y=1cm,
    >=Stealth,
    font=\footnotesize,
    box/.style={
        draw,
        rounded corners,
        minimum width=1.8cm,
        minimum height=0.85cm,
        align=center
    }
]
\node[box] (s) at (0,1.8) {$s$\\source};
\node[box] (X) at (0,0) {$X$\\visual};
\node[box] (H) at (0,-1.8) {$H$\\semantic};

\node[box] (A) at (4.0,1.8) {$A_s$\\readout};
\node[box] (y) at (4.0,0) {$y^{obs}$\\label};
\node[box] (b) at (4.0,-1.8) {$b_s$\\preference};

\draw[->] (s.east) -- node[midway,above] {\scriptsize taxonomy} (A.west);
\draw[->] (s.south) -- node[midway,left] {\scriptsize selection} (X.north);
\draw[->] (X.south) -- node[midway,left] {\scriptsize representation} (H.north);

\draw[->] (A.south) -- node[midway,right] {\scriptsize structural bias} (y.north);
\draw[->] (b.north) -- node[midway,right] {\scriptsize preference bias} (y.south);

\draw[->] (H.north east) -- node[pos=0.3, below=1pt] {\scriptsize causality} (y.south west);

\draw[->] (s.south east) to[bend right=18]
node[pos=0.48,below=2pt] {\scriptsize guideline} (b.north west);
\end{tikzpicture}
\vspace{-1mm}
\caption{
Causal view of subjective supervision in SLU.
The source $s$ affects both the visual distribution and the annotation mechanism.
We separate semantic evidence $H=\phi(X)$ from source-dependent annotation effects:
a taxonomy/readout operator $A_s$ and a preference term $b_s$.
}
\label{sup:fig:causal_graph_main}
\vspace{-2mm}
\end{figure}

We begin with a causal view of source-dependent annotations in SLU.
Let $s \in \mathcal{S}$ denote the data source, $X$ the raw visual input, and
$H=\phi(X)$ a semantic representation.
We analyze $H$ rather than raw $X$ so that annotation subjectivity is separated from low-level source-specific appearance.

\paragraph{Structural mask}

Let $\mathcal{Y}$ be a universal label vocabulary containing every category any expert might use (e.g.\ the union of all shot-scale taxonomies).  
Source~$s$ exposes only a subset of this vocabulary.  We encode this as a \emph{structural mask}
\begin{equation}\label{eq:mask}
  C_s(y)\;\in\;\{0,1\},
\end{equation}
where $C_s(y)=1$ if and only if label~$y$ is available in source~$s$'s taxonomy.  For example, a coarse shot-scale source may set $C_s(\text{medium close-up})=0$, forcing annotators to merge that category into a neighbouring label.
Structural masks capture the \emph{structural bias} arrow $s\to A_s$ in Figure~\ref{sup:fig:causal_graph_main}: the source determines which readout categories exist.

\paragraph{Preference offset.}
Even when two sources share the same label, their annotators may exhibit different tendencies—e.g.\ one team calls ambiguous framings ``medium shot'' while another leans toward
``medium long.''  We model this as a source-dependent \emph{preference offset}
\begin{equation}\label{eq:pref}
  b_s(y)\;\in\;\mathbb{R},
\end{equation}
where a positive $b_s(y)$ indicates that source~$s$ favours label~$y$ beyond what the visual evidence warrants. This corresponds to the \emph{preference bias} arrow
$s\to b_s$ in Figure~\ref{sup:fig:causal_graph_main}.

\paragraph{Source-invariant semantic score.}
The quantity we wish to learn is the \emph{source-invariant semantic compatibility} $f^{*}(H,y)$, which captures how well label~$y$ describes the visual content represented by~$H$.
This corresponds to the arrow $X\to H$ in Figure~\ref{sup:fig:causal_graph_main}.

\paragraph{Label generation model.}
Combining these three components via a log-linear model, the probability that source~$s$ assigns label~$y$ given semantic representation~$H$ is
\begin{equation}\label{eq:label-gen}
  P_s(y\mid H)
  \;=\;
  \frac{
    \exp\bigl(f^{*}(H,y)+b_s(y)\bigr)\,C_s(y)
  }{
    Z_s(H)
  },
\end{equation}
where the partition function sums over available labels only:
\begin{equation}\label{eq:partition}
  Z_s(H)
  \;:=\;
  \sum_{y'\in\mathcal{Y}}
    \exp\bigl(f^{*}(H,y')+b_s(y')\bigr)\,C_s(y').
\end{equation}

\paragraph{Takeaway.}
Equation~\eqref{eq:label-gen} makes explicit that observed labels conflate three distinct effects: semantic evidence ($f^{*}$), structural availability ($C_s$), and annotator preference ($b_s$).  
\emph{A model trained on a single source has no mechanism to separate them.}

\begin{remark}[Scope of the causal formulation]
\label{sup:rem:causal_scope_main}
Figure~\ref{sup:fig:causal_graph_main} and
Equation~\eqref{eq:label-gen}
do not claim that all cross-source differences are removed by the representation $H$.
Instead, they isolate the annotation-side mechanism that our theory studies:
source-dependent support/readout and source-dependent preference.
\end{remark}

\paragraph{Single-source non-identifiability.}
With a single source~$s_0$, the decomposition in \eqref{eq:label-gen} is fundamentally non-unique.  
Define
$\tilde{f}(H,y):=f^{*}(H,y)+b_{s_0}(y)$ and
$\tilde{b}(y):=0$; then
\begin{equation}\label{eq:nonident}
  P_{s_0}(y\mid H)
  \;=\;
  \frac{
    \exp\bigl(\tilde{f}(H,y)+\tilde{b}(y)\bigr)\,C_{s_0}(y)
  }{
    Z_{s_0}(H)
  },
\end{equation}
which is identical to~\eqref{eq:label-gen}.  More generally, for any vector $\delta(y)$, the pair
$(f^{*}(H,y)-\delta(y),\; b_{s_0}(y)+\delta(y))$ produces the same likelihood.  Therefore, the preference bias is absorbed into the learned scoring function, and the structural mask $C_{s_0}$ is invisible: the model never observes labels outside $\mathcal{Y}_{s_0}$, so it cannot detect
which canonical categories are missing.

\subsection{Why Large-Scale Multi-Source Data Is Necessary}
\label{sup:subsec:theory_data}

We now show that \textit{multi-source} pooling resolves the identifiability failure above, and that \textit{large-scale data} is subsequently needed to stabilise the resulting estimates.

\subsubsection{Multi-Source Pooling Makes Bias Observable}

When all sources are pooled into a single training set, each example $(H,y)$ is drawn by implicitly selecting a source~$s$ and then sampling $y\sim P_s(y\mid H)$.  The training
conditional over the pooled data is the exact mixture
\begin{equation}\label{eq:pool}
  P_{\mathrm{tr}}(y\mid H)
  \;=\;
  \sum_{s\in\mathcal{S}}\alpha_s(H)\,P_s(y\mid H),
\end{equation}
where $\alpha_s(H)=P(s\mid H)$.

We introduce two regularity assumptions to extract the dominant structure from this
mixture.

\begin{assumption}[Approximate source balance]
\label{ass:balance}
There exist constants $\pi_s\ge0$ with $\sum_s\pi_s=1$ and residuals $\delta_s(H)$ such that
\begin{equation}\label{eq:balance}
\begin{aligned}
  \alpha_s(H) &= \pi_s+\delta_s(H), \\
  \textstyle\sum_s\delta_s(H) &= 0, \\
  \textstyle\sum_s|\delta_s(H)| &\le \eta.
\end{aligned}
\end{equation}
\end{assumption}
\noindent\textit{Intuition.}  Knowing~$H$ does not strongly predict which source the example comes from.
This is plausible when sources cover overlapping visual domains.  The parameter $\eta$ measures how much visual distributions differ across sources.
\begin{assumption}[Bounded normalizer distortion]
\label{ass:normalizer}
There exist source-specific scalars~$a_s$, a shared function
$h(H)$, and residuals $\epsilon_s(H)$ such that
\begin{equation}\label{eq:normalizer}
\begin{aligned}
  \log Z_s(H) &= a_s + h(H) + \epsilon_s(H), \\
  \mathbb{E}_H\bigl[\epsilon_s(H)^2\bigr] &\le \varepsilon^2.
\end{aligned}
\end{equation}
\end{assumption}

\noindent\textit{Intuition.}  The partition functions across sources share a common $H$-dependent component, differing mainly by a constant shift~$a_s$.  
This holds when available label sets overlap substantially.

\begin{proposition}[Dominant source-induced prior]
\label{prop:prior}
Under Assumptions~\ref{ass:balance}--\ref{ass:normalizer},
\begin{equation}\label{eq:prior}
  P_{\mathrm{tr}}(y\mid H)
  \;\propto\;
  \exp\bigl(f^{*}(H,y)+\kappa(y)+r_y(H)\bigr),
\end{equation}
where
\begin{equation}\label{eq:kappa}
  \kappa(y)
  \;=\;
  \log\sum_{s\in\mathcal{S}}
    \pi_s\,C_s(y)\,\exp\bigl(b_s(y)-a_s\bigr)
\end{equation}
is a composite label prior, and the residual satisfies
$\mathbb{E}_H[|r_y(H)|]\le c_1\varepsilon+c_2\eta$.
\end{proposition}

\begin{proof}
Substituting~\eqref{eq:label-gen} into~\eqref{eq:pool} and
factoring out $\exp(f^{*}(H,y))$:
\begin{equation}\label{eq:pf-step1}
\begin{split}
  P_{\mathrm{tr}}(y\mid H)
  &= \exp\bigl(f^{*}(H,y)\bigr) \\
  &\quad \sum_s\alpha_s(H)\,\frac{C_s(y)\exp(b_s(y))}{Z_s(H)}.
\end{split}
\end{equation}
Applying Assumption~\ref{ass:normalizer},
$Z_s(H)^{-1}=\exp(-a_s-h(H)-\epsilon_s(H))$, so
\begin{align}
  P_{\mathrm{tr}}(y\mid H)&= \exp\bigl(f^{*}(H,y)-h(H)\bigr)\nonumber\\
  &\qquad\times
     \sum_s\alpha_s(H)\,C_s(y)\,
     \exp\bigl(b_s(y)\\
    &-a_s-\epsilon_s(H)\bigr).
  \label{eq:pf-step2}
\end{align}
Applying Assumption~\ref{ass:balance} and separating the
$\epsilon_s$ terms via $\exp(-\epsilon_s)=1-\epsilon_s+O(\epsilon_s^2)$:
\begin{equation}
  \sum_s(\cdots)
  \;=\;
  \underbrace{
    \sum_s\pi_s\,C_s(y)\,\exp(b_s(y)-a_s)
  }_{M(y)}
  \;+\; R_y(H),
  \label{eq:pf-step3}
\end{equation}
where $\mathbb{E}_H[|R_y(H)|]\le M(y)(c_1\varepsilon+c_2\eta)$.
Writing $M(y)+R_y(H)=M(y)\exp(r_y(H))$ with
$r_y(H):=\log(1+R_y(H)/M(y))$ and noting that $\kappa(y)=\log M(y)$
yields~\eqref{eq:prior}.  The bound on $r_y$ follows from
local Lipschitz continuity of $\log(1+x)$.
\end{proof}

The composite prior~$\kappa(y)$ is the central object.  We now state three consequences that show how multi-source pooling resolves the single-source identifiability failure.

\begin{corollary}[Structural coverage recovery]
\label{cor:coverage}
$\kappa(y)=-\infty$ if and only if $C_s(y)=0$ for \emph{all} $s\in\mathcal{S}$.  
Otherwise $\kappa(y)\ge\log(\pi_{s^*}\exp(b_{s^*}(y)-a_{s^*}))>-\infty$,
where $s^*$ is the source that maximises the corresponding summand in~\eqref{eq:kappa}.
\end{corollary}

\begin{proof}
$M(y)=\sum_s\pi_s C_s(y)\exp(b_s(y)-a_s)$ is a sum of
non-negative terms.  It equals zero iff every term is zero,
which occurs iff $C_s(y)=0$ for all~$s$ (since $\pi_s>0$ and
the exponential is strictly positive).
\end{proof}

\noindent\textbf{Interpretation.}
A single source~$s_0$ with $C_{s_0}(y)=0$ permanently hides label~$y$.  Multi-source pooling recovers it whenever \emph{any} source exposes it.  The accessible vocabulary grows
to $\mathcal{Y}_{\mathrm{pool}}=\bigcup_s\{y:C_s(y)=1\}$.

\begin{corollary}[Preference dilution]
\label{cor:dilution}
Define the single-source prior $\kappa_s(y):=b_s(y)-a_s$.  Then
\begin{equation}\label{eq:lse-sandwich}
\begin{split}
  &\max_s\bigl\{\kappa_s(y)+\log\pi_s\bigr\}
  \le \kappa(y) \\
  &\le \max_s\bigl\{\kappa_s(y)+\log\pi_s\bigr\}+\log|\mathcal{S}|.
\end{split}
\end{equation}
\end{corollary}

\begin{proof}
This is the standard log-sum-exp sandwich: the lower bound selects the largest summand in $M(y)$; the upper bound follows because there are at most $|\mathcal{S}|$ terms.
\end{proof}

\noindent\textbf{Interpretation.}
For a single source, the effective prior is $b_s(y)-a_s$,
which can be arbitrarily skewed.  Under pooling, $\kappa(y)$ is
a soft-max over all source-specific priors.  Extreme
preferences from any single source are tempered: if
source~$s_1$ strongly favours label~$y$ but the remaining
sources do not, $\kappa(y)$ will be moderate.  As $|\mathcal{S}|$
grows with diverse sources, the prior increasingly reflects
cross-source consensus.

\begin{corollary}[Cross-source contrast identifies
preference differences]
\label{cor:contrast}
Let sources $s_1,s_2$ both expose labels $y$ and $y'$
(i.e.\ $C_{s_1}(y)=C_{s_2}(y)=C_{s_1}(y')=C_{s_2}(y')=1$).
Then the cross-source log-odds difference
\begin{equation}\label{eq:contrast}
\begin{split}
  &\log\frac{P_{s_1}(y\mid H)}{P_{s_1}(y'\mid H)} - \log\frac{P_{s_2}(y\mid H)}{P_{s_2}(y'\mid H)} \\
  &\quad = \Delta b_{s_1}(y,y')-\Delta b_{s_2}(y,y'),
\end{split}
\end{equation}
where $\Delta b_s(y,y'):=b_s(y)-b_s(y')$.
\end{corollary}

\begin{proof}
The log-odds ratio for source~$s$ is
$\log\frac{P_s(y\mid H)}{P_s(y'\mid H)}
=(f^{*}(H,y)-f^{*}(H,y'))+(b_s(y)-b_s(y'))$.
The source-invariant semantic term cancels in the difference
across sources.
\end{proof}

\noindent\textbf{Interpretation.}
In Section~\ref{sup:subsec:theory_formulation} we showed that a single
source confounds $f^{*}$ with~$b_s$.  Here, the cross-source
log-odds difference is a \emph{pure function of preference
bias}---the semantic score has been eliminated.  This is a
directly observable quantity, enabling detection and
estimation of annotation bias.  It requires at least two
sources with overlapping but non-identical preference vectors.

\paragraph{Minimum diversity condition.}
Corollaries~\ref{cor:coverage}--\ref{cor:contrast} jointly
require that the source--label interaction matrix
$\mathbf{B}=[C_s(y)\exp(b_s(y))]_{s,y}$ has rank $\ge 2$, which
generically holds when $|\mathcal{S}|\ge 2$ with different
annotation mechanisms.  This fails for a single source
($\operatorname{rank}\le 1$), confirming the non-identifiability result
in Section~\ref{sup:subsec:theory_formulation}.

\subsubsection{Large-Scale Data Stabilises Bias Estimation}

Proposition~\ref{prop:prior} and
Corollaries~\ref{cor:coverage}--\ref{cor:contrast} establish
that multi-source pooling makes the composite prior
$\kappa(y)$ observable.  We now show that the quality of this
prior improves as the number of sources and data scale grow.
The key observation is that $\kappa(y)=\log M(y)$ where
$M(y)=\sum_s\pi_s C_s(y)\exp(b_s(y)-a_s)$ is itself a
weighted average over sources.  Thus, the stability
of~$\kappa$ is directly governed by concentration of this
average.

\begin{assumption}[Exchangeable sources with bounded mechanisms]
\label{ass:exch}
The annotation mechanisms $(C_s,b_s)$ are drawn exchangeably
from a population of possible annotation protocols such that the
preference offsets are uniformly bounded:
$|b_s(y)|\le B$ for all $s\in\mathcal{S}$ and
$y\in\mathcal{Y}$.
\end{assumption}

\noindent\textit{Intuition.}  This says that no single source
occupies a privileged structural role---each represents one
draw from a space of expert annotation conventions.  The
boundedness condition is mild: it requires that no source has
an infinitely strong preference for or against any label.

Under uniform mixing weights $\pi_s=1/|\mathcal{S}|$, define
the population-level quantity that $M(y)$ estimates:
\begin{equation}\label{eq:Mbar}
  \bar{M}(y)
  \;:=\;
  \mathbb{E}_{s}\bigl[C_s(y)\exp(b_s(y)-a_s)\bigr],
\end{equation}
and correspondingly
$\bar\kappa(y):=\log\bar{M}(y)$.  This is the prior that
would emerge from pooling infinitely many sources.

\begin{proposition}[Source diversity stabilises the prior]
\label{prop:scale}
Under Assumption~\ref{ass:exch}, for each label
$y\in\mathcal{Y}$ with $\bar{M}(y)>0$ and for any
$\delta>0$, with probability at least $1-\delta$,
\begin{equation}\label{eq:kappa-conc}
  \bigl|\kappa(y)-\bar\kappa(y)\bigr|
  \;\le\;
  \frac{e^{2B}}{\bar{M}(y)}
  \sqrt{\frac{\log(2/\delta)}{2|\mathcal{S}|}}.
\end{equation}
Applying a union bound over labels, with probability at
least $1-\delta$,
\begin{equation}\label{eq:kappa-unif}
  \|\kappa-\bar\kappa\|_\infty
  \;\le\;
  \frac{e^{2B}}{M_{\min}}
  \sqrt{\frac{\log(2|\mathcal{Y}|/\delta)}{2|\mathcal{S}|}},
\end{equation}
where
$M_{\min}:=\min_{y:\bar{M}(y)>0}\bar{M}(y)$.
\end{proposition}

\begin{proof}
Each summand in $M(y)$ is
$X_s(y):=\frac{1}{|\mathcal{S}|}C_s(y)\exp(b_s(y)-a_s)$,
which satisfies $0\le X_s(y)\le e^{2B}/|\mathcal{S}|$ under
Assumption~\ref{ass:exch}.  By Hoeffding's inequality applied
to the sum $M(y)=\sum_s X_s(y)$:
\begin{equation}\label{eq:hoeff}
  \Pr\bigl[\,|M(y)-\bar{M}(y)|\ge t\,\bigr]
  \;\le\;
  2\exp\Bigl(-\frac{2t^2|\mathcal{S}|}{e^{4B}}\Bigr).
\end{equation}
Setting the right-hand side equal to $\delta$ and solving
for~$t$ gives
\begin{equation}\label{eq:M-bound}
  |M(y)-\bar{M}(y)|
  \;\le\;
  e^{2B}\sqrt{\frac{\log(2/\delta)}{2|\mathcal{S}|}}
\end{equation}
with probability $\ge 1-\delta$.
Since $\kappa(y)=\log M(y)$ and $\bar\kappa(y)=\log\bar{M}(y)$,
the mean value theorem gives
$|\kappa(y)-\bar\kappa(y)|\le|M(y)-\bar{M}(y)|/\bar{M}(y)$
(using that $|\frac{d}{dx}\log x|=1/x$ and that $M(y)$
stays close to $\bar{M}(y)$).  
Substituting~\eqref{eq:M-bound} yields~\eqref{eq:kappa-conc}.  
The uniform bound~\eqref{eq:kappa-unif} follows by a union bound over
$|\mathcal{Y}|$ labels.
\end{proof}

\noindent\textbf{Interpretation.}
The denominator
$\bar{M}(y)\sqrt{|\mathcal{S}|}$ shows that stability improves with both \emph{source diversity}
($|\mathcal{S}|$~growing) and \emph{label coverage} ($\bar{M}(y)$ being large, i.e.\ many sources expose label~$y$).  Labels that only a few sources cover have small~$\bar{M}(y)$ and thus require more sources to stabilize.

\paragraph{Overall takeaway.}
Our analysis supports the following picture: (i)~a single source cannot separate semantic evidence from annotation bias (Section~\ref{sup:subsec:theory_formulation});
(ii)~multi-source pooling exposes a dominant composite prior~$\kappa(y)$ that aggregates structural and preference biases, recovering hidden labels
(Corollary~\ref{cor:coverage}), diluting extreme preferences (Corollary~\ref{cor:dilution}), and creating observable contrasts that isolate bias from semantics (Corollary~\ref{cor:contrast});
(iii)~the quality of this prior improves at rate
$O(1/\sqrt{|\mathcal{S}|})$ with source diversity and label diversity (Proposition~\ref{prop:scale}) making balanced all-round SLU performance contingent on both sufficient source breadth and data depth.
\label{sec:appendix}


\section{Related Work}
\label{app:related_work}

\paragraph{Film theory and cinematographic foundations.}
Shot language is rooted in classical film theory and cinematography practice. Foundational texts in film studies systematically describe how composition, shot scale, camera viewpoint, movement, lighting, and editing jointly shape narrative structure, emotion, and visual style~\cite{mascelli1965five,block2020visual}. These works provide the conceptual basis for treating shot language as a structured and film-grounded semantic space rather than as a loose collection of low-level visual attributes. Recent efforts have also begun to formalize cinematographic knowledge into more explicit taxonomies and evaluation protocols for computational models~\cite{chatterjeestable}.
Our formulation follows this line of thinking. We organize SLU around six high-level dimensions: \emph{Composition}, \emph{Coverage}, \emph{Viewpoint}, \emph{Motion}, \emph{Lighting}, and \emph{Cuts}, while preserving the expert-defined category boundaries used by different data sources.

\paragraph{Shot language benchmarks and movie-focused datasets.}
With the rise of vision-language models, several benchmarks have started to evaluate model capability on film-specific visual understanding. ShotBench~\cite{shotbench} studies expert-level cinematic understanding through question answering over multiple shot-language aspects; CameraBench~\cite{camerabench} focuses on camera motion understanding in videos; CineTechBench~\cite{cinetechbench} expands the scope to a broader set of cinematographic techniques such as shot scale, composition, lighting, and movement; MovieCuts~\cite{moviecuts} focuses on edit-transition understanding and cut-type recognition; RefineShot~\cite{wu2025refineshot} further revisits the design of cinematic evaluation by emphasizing more consistent and skill-oriented task construction. At the dataset level, CineScale~\cite{savardi2021cinescale} and CineScale2~\cite{savardi2023cinescale2} provide large-scale annotations of cinematic shot scale and camera angle/height in movies, while movie-oriented resources such as MovieNet~\cite{huang2020movienet}, MovieShots~\cite{rao2020unified} provide complementary annotations for shot type, scale, and motion in long-form film content. These resources have substantially advanced SLU, but they remain highly heterogeneous in taxonomy granularity, label space, supervision format, and temporal scope across datasets. In contrast, our goal is not to introduce yet another single benchmark, but to unify fragmented human-labeled resources into one suite that supports systematic multi-task training and both in-domain (ID) and out-of-domain (OOD) evaluation.

\paragraph{Multimodal instruction tuning and QA unification.}
Our data construction is also related to multimodal instruction tuning, which has become a standard recipe for adapting pretrained vision-language models to diverse downstream tasks. Prior work shows that high-quality multimodal instruction data can substantially improve generalization and usability~\cite{liu2023visual}, and later studies explore more structured instruction organization through conditional LoRA selection, interaction grouping, or instruction-level routing~\cite{dai2023instructblip,huang2023lorahub}. Surveys also further summarize the growing design space of multimodal instruction tuning and data construction~\cite{li2024multimodal}. Similar in spirit, our suite unifies heterogeneous supervision into a shared QA-style interface. However, our focus is not on generic instruction following. Instead, we convert single-label classification, multi-label prediction, binary QA, and caption-like supervision into a common question-answer formulation tailored to shot-language reasoning, with prompts and option sets aligned to expert terminology and source-specific category boundaries.

\paragraph{Parameter-efficient adaptation of vision-language models.}
Our adaptation strategy is closely related to the literature on parameter-efficient fine-tuning (PEFT) for vision-language models. A large body of work has proposed adapters, cross-modal adaptation modules, and LoRA-based tuning strategies to transfer pretrained VLMs to downstream tasks with a small number of trainable parameters~\cite{hu2022lora,pfeiffer2021adapterfusion,houlsby2019parameter}. Subsequent empirical studies and surveys analyze the broader PEFT design space for both VLMs and LLMs~\cite{lialin2023scaling,han2024parameter,wang2025parameter,ostapenko2024towards}.
Most of these works primarily ask how to achieve stronger transfer with lower cost.
By contrast, our RQ1 asks a more task-specific question for SLU: which part of a pretrained VLM actually matters most? Our experiments show that adapting the language model, or the language model together with the multimodal connector, already recovers most of the achievable gains, and can even outperform broader adaptation. This complements the PEFT literature by suggesting that, for SLU, the main bottleneck is often the alignment of language-side decision boundaries and professional terminology rather than insufficient raw visual representation.

\paragraph{Multi-source transfer, task interference, and data mixing.}
Our work is also related to multi-source transfer learning and multi-task optimization, where a central challenge is how to balance positive transfer against task interference.
Prior work on multi-source domain adaptation highlights that different sources may share latent semantic content while differing in observed distributions or annotation mechanisms~\cite{long2015learning,peng2019moment,wen2020domain}.
Related studies on data mixing and instruction mixture design further show that the composition of training data itself can be treated as an optimization problem~\cite{gong2021mdalu,zhao2020object,kim2022learning}.
These ideas are closely connected to our empirical study of cross-dimensional transfer in SLU.
Our transfer matrix reveals strong heterogeneity and asymmetry: some dimensions are strong donors but weak recipients, while others benefit more from incoming supervision than they provide to others.
Motivated by this, UniShot can be viewed as a transfer-aware data mixing strategy that dynamically reweights source dimensions using both empirical transfer priors and validation feedback, aiming to improve overall balance within a single universal model.

\paragraph{Mixture-of-experts, LoRA mixtures, and routing.}
Another related line of work uses modular experts to reduce interference while preserving specialization.
Sparse MoE architectures, multimodal expert mixtures, and LoRA-based expert compositions have been explored for multi-task or multimodal large models~\cite{kendall2018multi,chen2018gradnorm,yu2020gradient,xie2023doremi,liu2024regmix}.
Recent works further investigate conditional mixtures of LoRA, curriculum-based LoRA experts, and prompt- or instruction-routed experts for more efficient specialization~\cite{dai2023instructblip,diao2025climb,fedus2022switch,du2022glam}.
Our AgentShots is related to these methods in spirit, but differ in both granularity and design goal.
Instead of building a large end-to-end sparse architecture, we keep a shared frozen backbone and learn one lightweight expert per high-level SLU dimension.
Routing is performed at the task level through prompts and expert descriptions, yielding top-1 specialist selection with the same activated parameter scale as UniShot at inference time.
This makes AgentShots a lightweight specialist design tailored to universal shot language understanding rather than a generic MoE framework.

\paragraph{Subjective supervision and annotation bias.}
Finally, our theoretical analysis connects to prior work on source-dependent annotation bias and subjective supervision.
In SLU, labels are often not purely objective: different datasets may reflect different taxonomies, stylistic priors, and annotation boundaries.
This issue is related to causal views of multi-source learning, which emphasize that source identity can affect not only the input distribution but also the labeling mechanism itself~\cite{long2015learning,peng2019moment}.
It is also related to research on annotator bias, disagreement, and context-dependent labeling, which argues that supervision should not always be treated as coming from a single source of ground truth~\cite{zhou2022mixture,muqeeth2024learning,zhao2024loraretriever}.
Our formulation makes this perspective explicit by separating source effects on visual distribution and annotation preference, and by analyzing when cross-source mixing can help identify stable semantics under subjective supervision.
This theoretical view provides additional support for our empirical finding that diversity across heterogeneous sources is more valuable than merely scaling data within a single source.

\section{Limitations}
\label{app:limitations}

While SLU-SUITE provides a comprehensive benchmark for visual shot language understanding, filmmaking is inherently an audio-visual medium. Currently, our framework and the proposed universal models (UniShot and AgentShots) focus exclusively on the visual modality. This scope is primarily motivated by two factors. First, leading open-weight vision-language models adopted in our study (e.g., Qwen3-VL) natively process visual inputs and lack direct support for acoustic encoding. Second, while sound design is crucial in film editing---such as in audio match cutting~\cite{fedorishin2024audio}---these acoustic techniques generally fall outside the primary scope of traditional visual cinematography analysis (e.g., framing, camera movement, and lighting). Furthermore, large-scale, high-quality, and open-source datasets dedicated to expert-level cinematic audio transitions remain scarce. Future work could extend our suite by incorporating emerging omni-modal foundation models to jointly analyze soundscapes and visual shot language.

\section{Ethics Statement}
SLU-SUITE is curated from 11 existing public and academic datasets. We do not collect new data from human subjects; instead, we follow the original licenses and terms of use of the source datasets, which were already anonymized and released for research purposes.

\end{document}